\documentclass{article}

\usepackage[utf8]{inputenc} 
\usepackage[T1]{fontenc}    
\usepackage[backref=page]{hyperref}       
\usepackage{url}            
\usepackage{booktabs}       
\usepackage{amsfonts}       
\usepackage{nicefrac}       
\usepackage{microtype}      
\usepackage[dvipsnames]{xcolor}         
\usepackage{pifont}
\usepackage{subcaption}
\usepackage{fullpage}
\usepackage{natbib}
\usepackage{enumitem}
\usepackage{amsmath}
\usepackage{amsthm}
\usepackage{wrapfig}
\usepackage{diagbox}

\usepackage{enumitem}

\def\Eqref#1{Equation~\ref{#1}}

\usepackage[capitalize,noabbrev]{cleveref}
\usepackage{mleftright}

\usepackage{graphicx}       

\usepackage{tikz}
\usepgflibrary{shapes.geometric}
\definecolor{gold(metallic)}{rgb}{0.83, 0.69, 0.22}

\usepackage{dsfont}
\theoremstyle{plain}
\newtheorem{theorem}{Theorem}[section]
\newtheorem{proposition}[theorem]{Proposition}
\newtheorem{lemma}[theorem]{Lemma}
\newtheorem{corollary}[theorem]{Corollary}
\theoremstyle{definition}
\newtheorem{definition}[theorem]{Definition}
\theoremstyle{remark}
\newtheorem{remark}[theorem]{Remark}

\newcommand{\innerprod}[2]{\left\langle #1,#2 \right\rangle}
\newcommand{\norm}[1]{\lVert #1 \rVert}

\newcommand{\cL}{\mathcal{L}}
\newcommand{\w}{\mathbf{w}}
\newcommand{\x}{\mathbf{x}}

\newcommand{\thetab}{\bm{\theta}}

\renewcommand{\tilde}{\widetilde}

\DeclareMathOperator*{\argmin}{\arg\!\min}

\usepackage{bm}
\usepackage{nicematrix}

\renewcommand*{\backref}[1]{}
\renewcommand*{\backrefalt}[4]{%
    \ifcase #1 (Not cited.)%
    \or        (Cited on page~#2.)%
    \else      (Cited on pages~#2.)%
    \fi}

\title{Flavors of Margin: Implicit Bias of Steepest Descent in Homogeneous Neural Networks}
\date{}
\author{
\begin{tabular}{c}
Nikolaos Tsilivis$^1$,
Eitan Gronich$^{2}$,
Julia Kempe$^{1, 3},$
Gal Vardi$^{2}$
\\
\\
\normalsize{$^1$New York University, $^2$Weizmann Institute of Science, $^3$Meta FAIR}\\
\\
\normalsize{\texttt{nt2231@nyu.edu, eitan.gronich@weizmann.ac.il, kempe@meta.com, gal.vardi@weizmann.ac.il}}
\end{tabular}
}

\begin{document}

\maketitle

\def\doublecolumn{0}


\begin{abstract}
We study the implicit bias of the general family of steepest descent algorithms with infinitesimal learning rate in deep homogeneous neural networks. We show that: (a) an algorithm-dependent geometric margin starts increasing once the networks reach perfect training accuracy, and (b) any limit point of the training trajectory corresponds to a KKT point of the corresponding margin-maximization problem. We experimentally zoom into the trajectories of neural networks optimized with various steepest descent algorithms, highlighting connections to the implicit bias of popular adaptive methods (Adam and Shampoo).
\end{abstract}

\section{Introduction}

Overparameterized neural networks excel in many natural supervised learning applications. A theory that aims to explain their strong generalization performance places optimization at the forefront: in problems where many candidate models are available, the optimization algorithm implicitly selects well-generalizing ones \citep{NTS15}. The term “implicitly” indicates that the loss/objective function does not explicitly favor simple, well-generalizing solutions, rather this occurs due to the properties of the optimization algorithm. Most existing theoretical results on this \textit{implicit bias of optimization} demonstrate, to some extent, that gradient descent in overparameterized problems biases the solution to be the \textit{simplest}, in terms of the lowest possible $\ell_2$ norm of the weights \citep{Sou+18, JiTe19, JiTe20, LyLi20, Nac+19}.

Simplicity, however, is a term that depends on the setting.
For instance, in logistic regression with many irrelevant features, an $\ell_1$-regularized solution is simpler than an $\ell_2$-regularized one \citep{Ng04}.
Moreover, in contemporary deep learning, Adam \citep{KiBa15}, AdamW \citep{LoHu19}, and related optimization algorithms are preferred for language modeling \citep{Zha+20}, and their implicit bias might be better suited for such applications than gradient descent.
It is therefore important to understand the types of solutions favored by optimization algorithms beyond (stochastic) gradient descent, in order to address the current (and future) applications of deep learning.

In this work, we contribute to this line of research by studying the large family of \textit{steepest descent} algorithms (Equation \ref{eq:steep_desc}) with respect to an arbitrary norm $\lVert \cdot \rVert$ in deep, non-linear, homogeneous neural networks. This class of methods extends gradient descent to optimization geometries other than the Euclidean, allowing the update rule to operate under a different norm. It includes coordinate descent (which has strong ties to boosting \citep{Mas+99}), sign gradient descent (which is closely related to Adam \citep{Kun+23}) and spectral steepest descent (which is similar to Shampoo~\citep{Gup+18}) as special cases.

\textbf{Our contributions.} We provide a unifying, rigorous analysis of any steepest descent algorithm in classification settings with locally-Lipschitz, homogeneous neural networks trained using an exponentially-tailed loss. Specifically, we focus on the late stage of training (after the network has achieved perfect training accuracy) in the limit of an infinitesimal learning rate.

Our first result characterizes the algorithm's tendency to increase an algorithm-dependent margin  (Theorem~\ref{thm:margin_monotonicity_main}): similar to prior work on gradient descent \citep{LyLi20}, we show that a \textit{soft} version of the geometric margin starts increasing immediately after fitting the training data.
We then turn our attention to the asymptotic properties of the algorithm.
As we show,
the limit points of training are along the direction of a Karush-Kuhn-Tucker (\textit{KKT}) point of the algorithm-dependent margin maximization problem (Theorem \ref{thm:main_result_main}). 

In total, these results establish (geometric) margin-maximization in \textit{any} steepest descent algorithm and significantly generalize prior results that concerned gradient descent only \citep{LyLi20,Nac+19}. They also generalize a previous version of the current paper, which only established bias towards KKT points in the case of algorithms whose squared norm is a smooth function~\citep{TVK25}. See the end of Section~\ref{ssec:main_result} for a technical discussion.

Finally, in Section~\ref{sec:expr}, we train neural networks with the three main steepest descent algorithms (gradient descent, sign gradient descent and coordinate descent). We perform experiments in: (a) teacher-student tasks, to elucidate the theoretical findings and assess the connection between implicit bias and generalization and (b) image classification tasks, to study the relationship between Adam~\citep{KiBa15}, Shampoo~\citep{Gup+18} and steepest descent algorithms.

\subsection{Related Work}
There have been numerous works studying the implicit bias of optimization in supervised learning and their relationship to geometric margin maximization - see~\citet{Var23} for a survey. 

Steepest descent algorithms with respect to non-Euclidean geometries have been explored before, both in supervised (e.g.~\citet{NSS15,Lar+24}) and non-supervised (e.g.~\citet{Car+15}) machine learning problems.
The implicit bias of this family of optimization methods was first studied in generality in \citet{Gun+18} in the context of linear models for separable data, where margin maximization was established. Their proof is based on a result on Adaboost due to \citet{Tel13}. Our results generalize the analysis of steepest descent algorithms to any homogeneous neural network. Most related to our paper are the works of~\citet{Nac+19,LyLi20} and~\citet{JiTe20}. \citet{Nac+19} studied infinitesimal regularization and its connection to margin maximization in both homogeneous and non-homogeneous deep models, while also proving (directional) convergence of gradient descent to a first order point of an $\ell_2$-margin maximization problem for homogeneous models under strong technical assumptions. \citet{LyLi20,JiTe20}, whose theoretical setup we mainly follow, significantly weakened the assumptions, under which such a result holds, and \citet{LyLi20} further demonstrated the experimental benefits of margin maximization in terms of robustness. \citet{KuYa+23} generalized these results to a broader class of networks with varying degree of homogeneity, while \citet{Cai+24} analyzed non-homogeneous 2-layer networks trained with a large learning rate. \citet{VSS22} identified cases where the KKT points of the $\ell_2$ margin maximization problems are not (even locally) optimal.

The implicit bias of Adam~\citep{KiBa15} has been previously studied in~\citet{Wan+21,Wan+22} for homogeneous networks, where it is proven that it shares the same asymptotic properties as gradient descent ($\ell_2$ margin maximization).
Recently,~\citet{ZZC24} analyzed a version of the algorithm in linear models, without a numerical precision constant, which arguably better captures realistic training runs, and found bias towards $\ell_1$ margin maximization - the same bias as in the case of sign gradient descent. This makes us optimistic that insights from our analysis, which is connected to sign gradient descent (normalized steepest descent with respect to the $\ell_\infty$ norm), can shed light on the poorly understood implicit bias of Adam in deep neural networks. See also~\citet{ShLi24} for a recently established connection between AdamW \citep{LoHu19} and sign gradient descent. An additional motivation for studying steepest descent algorithms is in improving the robustness of deep neural networks: \citet{Tsi+24}, recently, provided experimental evidence and theoretical arguments that adversarially trained deep networks exhibit significant differences in their (robust) generalization error, depending on which steepest descent algorithm was used in training.

\section{Background}

\paragraph{Learning Setup.} We consider binary classification problems with deep, homogeneous, neural networks. Formally, let $\mathcal{S} = \left\{\mathbf{x}_i, y_i\right\}_{i=1}^m$ be a dataset of i.i.d.~points sampled from an unknown distribution $\mathcal{D}$ with $\mathbf{x}_i \in \mathbb{R}^d$ and $y_i \in \{\pm 1\}$ for all $i \in [m]$, and let $f(\cdot ; \thetab): \mathbb{R}^d \to \mathbb{R}$ denote a neural network parameterized by $\bm{\theta} \in \mathbb{R}^p$. The vector $\thetab$ contains all the parameters of the neural network, concatenated into a single vector. We study training under an exponential loss $\mathcal{L} (\thetab) = \sum_{i = 1}^m e^{- y_i f(\x_i ; \thetab)}$. We focus on this setting for simplicity in the main text, but our results generalize to more common losses, such as the logistic loss, as well as its multi-class generalization - the cross-entropy loss. See Section~\ref{ssec:gen_loss} for details and extensions of our main result.

\paragraph{Algorithms.} The family of steepest descent algorithms generalizes gradient descent to different optimization geometries, allowing the update rule to operate under an arbitrary norm (instead of the usual Euclidean one) \citep{BoVa14}. Formally, the update rule for \textit{steepest descent} with respect to a norm $\lVert \cdot \rVert$ is:
\begin{equation}\label{eq:steep_desc}
    \begin{split}
        \thetab_{t+1}  & = \thetab_{t}+\eta_t \Delta \thetab_t,\text{ where }\Delta \thetab_{t}\text{ satisfies }  \\
         \Delta \thetab_t & = \argmin_{\| \mathbf{u} \| \leq \| \nabla \cL(\thetab_t) \|_{\star}} \innerprod{\mathbf{u}}{\nabla \mathcal{L}(\thetab_t)}, 
    \end{split}
\end{equation}
where the \textit{dual} norm $\lVert \cdot \rVert_\star$ of $\lVert \cdot \rVert$ is defined as $\|\mathbf{z}\|_\star = \max_{\mathbf{v}} \{ \left|\innerprod{\mathbf{z}}{\mathbf{v}}\right|: \|\mathbf{v}\| = 1\}$ for any $\mathbf{z}$, and $\eta_t$ is a learning rate. Gradient descent can be derived from~\Eqref{eq:steep_desc} with $\lVert \cdot \rVert = \lVert \cdot \rVert_2$. See Appendix~\ref{sec:steep_adaptive} for details on how steepest descent algorithms are closely related to popular adaptive methods, such as Adam~\citep{KiBa15} and Shampoo~\citep{Gup+18}.

\paragraph{Assumptions \& Technical Points.} In order to formally allow for commonly used activation functions, such as the ReLU, we theoretically analyze loss landscapes that are not necessarily differentiable. That is, we consider Clarke's subdifferentials \citep{Cla75} in our analysis:
\begin{equation}
    \partial f := \mathrm{conv} \left\{ \lim_{k \to \infty} \nabla f(\x_k): \x_k \to \x, f \text{ differentiable at } \mathbf{x}_k \right\},
\end{equation}
where $\mathrm{conv} (\cdot)$ stands for the convex hull of a set. \newline
Furthermore, we analyze steepest descent in the limit of infinitesimal step size, i.e. \textit{steepest flow}:
\begin{equation}\label{eq:steep_flow}
    \frac{d \thetab}{d t} \in \left\{ \argmin_{\| \mathbf{u} \| \leq \|\mathbf{g}_t\|_\star} \innerprod{\mathbf{u}}{\mathbf{g}_t}: \mathbf{g}_t \in \partial \cL(\thetab_t) \right\}.
\end{equation}
This choice simplifies the analysis while still capturing the essence of the bias of the algorithms. Finally, we make the following assumptions:
\begin{enumerate}[label=(A\arabic*), ref=A\arabic*]
    \item Local Lipschitzness: For any $\x_i \in \mathbb{R}^d$, $f(\x_i; \cdot): \mathbb{R}^p \to \mathbb{R}$ is locally Lipschitz (and admits a chain rule - see Theorem~\ref{thm:chain_rule}). \label{ass:1}
    \item $L$-Homogeneity: We assume that $f$ is $L$-homogeneous in the parameters, i.e. $f(\cdot; c\thetab) = c^L f(\cdot; \thetab)$ for any $c>0$ and $\thetab$. \label{ass:2}
    \item Realizability: There is a $t_0 > 0$, such that $\mathcal{L}(\thetab_{t_0}) < 1$. \label{ass:3}
\end{enumerate}
Assumption (\ref{ass:1}) is a minimal assumption on the regularity of the network, while assumption (\ref{ass:2}) includes many commonly used architectures. For instance, ReLU networks with an arbitrary number of layers, but without bias terms and skip connections, satisfy (\ref{ass:1}),(\ref{ass:2}).
Assumption (\ref{ass:3}) ensures that the algorithm will succeed in classifying the training points and allows us to focus on what happens beyond that point of separation. Indeed, we are particularly interested in understanding the geometric properties of the model $f(\cdot ; \thetab_t)$ as $t \to \infty$ (at convergence) -- the \textit{implicit bias} of the learning algorithms.

\section{Theory}

We analyze the behavior of steepest descent algorithms in the late stage of training and study their geometric properties and how these relate to geometric, algorithm-specific, margins.

\subsection{Algorithm-Dependent Margin Increases}

\sloppy
In {\em linear} models, where $f(\x;\thetab) = \innerprod{\thetab}{\x}$, the concept of $\lVert \cdot \rVert_\star$-\textit{geometric margin}
\footnote{In this paper, we diverge from the established terminology when it comes to naming margins, by calling it $\lVert \cdot \rVert_\star$-geometric margin (instead of $\lVert \cdot \rVert$-geometric margin) when it is defined with respect to the $\lVert \cdot \rVert$ norm of the parameters. We believe this is proper, since the $\lVert \cdot \rVert_\star$-geometric margin in linear models maximizes the metric induced by the $\lVert \cdot \rVert_\star$ norm (and not its dual, $\lVert \cdot \rVert$).} 
, $\min_{i \in [m]}\frac{y_i \innerprod{\thetab}{\x_i}}{\|\thetab\|}$, plays a central and fundamental role in the analysis of the convergence of training \citep{Nov63} as well as in the generalization of the final model \citep{Vap98}.
Ideally, we would like to track a similar quantity when training general, homogeneous, non-linear networks $f(\cdot;\thetab)$ with steepest descent with respect to the $\lVert \cdot \rVert$ norm:
\begin{equation}\label{eq:geometric_margin}
    \gamma (\thetab) = \frac{\min_{i \in [m]} y_i f(\x_i; \thetab)}{\|\thetab\|^L} = \min_{i \in [m]} y_i f\left(\x_i; \frac{\thetab}{\|\thetab\|}\right),
\end{equation}
where recall that $L$ is the level of homogeneity of the model. As it turns out, it is easier to follow the evolution of the following, \textit{soft}, geometric margin:
\begin{equation}\label{eq:smooth_margin}
    \Tilde{\gamma}(\thetab) = - \frac{\log \mathcal{L}(\thetab)}{\|\thetab\|^L}.
\end{equation}
The characterisation of ``soft'' comes from the definition of ``softmax'' (a.k.a.~log-sum-exp), which is often used in machine learning. The same idea is used here to approximate the numerator of~\Eqref{eq:geometric_margin}. The soft margin $\Tilde{\gamma}(\thetab)$ is at most an additive $\log m$ away from $\gamma(\thetab)$ and converges to $\gamma(\thetab)$ as $t \to \infty$ - see Lemma~\ref{lem:smooth_hard_margin} and Corollary~\ref{cor:soft_hard_convergence}. 

We show next that, given the algorithm has reached a small value in the loss, the soft margin is non-decreasing. This theorem is similar to part of Lemma 5.1 in \citep{LyLi20}, which is the key lemma in their result. Our proof is admittedly simpler, avoiding a beautiful polar decomposition which was crucial in their analysis, yet, unfortunately, pertinent to the $\ell_2$ case only.
\begin{theorem}[Soft margin increases]\label{thm:margin_monotonicity_main}
For almost any $t > t_0$, it holds: 
$$\frac{d \log \Tilde{\gamma}}{d t} \geq L \left\lVert \frac{d \thetab}{d t} \right\rVert^2 \left( \frac{1}{L \cL(\thetab_t) \log\frac{1}{\cL(\thetab_t)}} - \frac{1}{\|\thetab_t\|\left\lVert \frac{d \thetab}{d t} \right\rVert}\right) \geq 0.$$
\end{theorem}
\begin{proof}[Simplified version]
    We present a proof for a simplified version of this theorem here, covering differentiable networks $f$, while we defer the full proof to Appendix~\ref{ssec:late_phase_proofs}. For differentiable losses, steepest flow corresponds to:
    \begin{equation}
        \frac{d \thetab}{d t} \in  \argmin_{\| \mathbf{u} \| \leq \|\nabla \cL(\thetab_t)\|_\star} \innerprod{\mathbf{u}}{\nabla \cL (\thetab_t)}.
    \end{equation}
    By the definition of the dual norm and chain rule, we have for any $t > 0$:
    \begin{equation}\label{eq:norms_equal_loss_descent_main}
        \left\lVert \frac{d \thetab}{d t} \right\rVert = \|\nabla \cL (\thetab_t)\|_\star \;\;\; \mathrm{and} \;\;\; \frac{d \cL (\thetab_t)}{d t} = - \left\lVert \frac{d \thetab}{d t} \right\rVert^2.
    \end{equation}

    Let $\mathbf{n}_t \in \partial \|\thetab_t\|$ (recall that a norm $\lVert \cdot \rVert$ might not be differentiable everywhere). For any $t > t_0$, we have:
    \begin{equation}\label{eq:margin_lemma_eq1_main}
        \begin{split}
            \frac{d \log \Tilde{\gamma}}{d t} & = \frac{d}{d t} \log \log \frac{1}{\cL (\thetab_t)} - L \frac{d}{d t} \log \|\thetab_t\| \\
            & = \frac{d}{d t} \log \log \frac{1}{\cL (\thetab_t)} - L \innerprod{\frac{\mathbf{n}_t}{\|\thetab_t\|}}{\frac{d \thetab}{d t}} \;\;\;\;\;\;\;\;\;\;\;\;\;\;\; (\text{Chain rule})  \\
            & \geq \frac{d}{d t} \log \log \frac{1}{\cL (\thetab_t)} - L \frac{\left\lVert \frac{d \thetab}{d t} \right\rVert}{\|\thetab_t\|} \;\;\;\;\; (\text{def. of dual norm and $\|\mathbf{n}_t \|_\star \leq 1$, Lemma~\ref{lem:subgrad_norm}}) \\
            & = - \frac{d \cL (\thetab_t)}{d t} \frac{1}{ \cL(\thetab_t) \log\frac{1}{\cL(\thetab_t)}} - L \frac{\left\lVert \frac{d \thetab}{d t} \right\rVert}{\|\thetab_t\|} \;\;\;\;\;\;\;\;\;\;\;\;\;\; (\text{Chain rule}) \\
            & = \left\lVert \frac{d \thetab}{d t} \right\rVert^2 \left( \frac{1}{ \cL(\thetab_t) \log\frac{1}{\cL(\thetab_t)}} - \frac{L}{\|\thetab_t\|\left\lVert \frac{d \thetab}{d t} \right\rVert}\right). \;\;\;\;\; (\text{\Eqref{eq:norms_equal_loss_descent_main}})
        \end{split}
    \end{equation}
    The first term inside the parenthesis can be related to the second one via the following calculation:
    \begin{equation}
        \begin{split}
            \innerprod{\thetab_t}{- \nabla \cL (\thetab_t)} & = \innerprod{\thetab_t}{\sum_{i = 1}^m e^{- y_i f(\x_i ; \thetab_t)} y_i \nabla f(\x_i; \thetab_t)} \\
            & = \sum_{i = 1}^m e^{- y_i f(\x_i ; \thetab_t)} y_i \innerprod{\thetab_t}{\nabla f(\x_i; \thetab_t)} \\
            & = L \sum_{i = 1}^m  e^{- y_i f(\x_i ; \thetab_t)} y_i f(\x_i ; \thetab_t),
        \end{split}
    \end{equation}
    where the last equality follows from Euler's theorem for homogeneous functions.
    Now, observe that this last term can be lower bounded as:
    \begin{equation}\label{eq:homog_ineq_main}
        \begin{split}
            \innerprod{\thetab_t}{- \nabla \cL (\thetab_t)} \geq L \sum_{i = 1}^m e^{- y_i f(\x_i ; \thetab_t)} \min_{i \in [m]} y_i f(\x_i ; \thetab_t) \geq L \cL (\thetab_t) \log \frac{1}{\cL (\thetab_t)},
        \end{split}
    \end{equation}
    where we used the fact $e^{- \min_{i \in [m]} y_i f(\x_i ; \thetab_t)} \leq \sum_{i = 1}^m e^{-y_i f(\x_i ; \thetab_t)} = \cL (\thetab_t)$. We have made the first term of~\Eqref{eq:margin_lemma_eq1_main} appear. By plugging~\Eqref{eq:homog_ineq_main} into~\Eqref{eq:margin_lemma_eq1_main}, we get:
    \begin{equation}
        \begin{split}
            \frac{d \log \Tilde{\gamma}}{d t} & \geq \left\lVert \frac{d \thetab}{d t} \right\rVert^2 \left( \frac{L}{\innerprod{\thetab_t}{- \nabla \cL (\thetab_t)}} - \frac{L}{\|\thetab_t\|\left\lVert \frac{d \thetab}{d t} \right\rVert}\right) \\
            & \geq \left\lVert \frac{d \thetab}{d t} \right\rVert^2 \left( \frac{L}{\|\thetab_t\| \|\nabla \cL (\thetab_t)\|_\star} - \frac{L}{\|\thetab_t\|\left\lVert \frac{d \thetab}{d t} \right\rVert}\right). \;\;\;\;\; (\text{definition of dual norm})
        \end{split}
    \end{equation}
    Noticing that $\|\nabla \cL (\thetab_t)\|_\star = \left\lVert \frac{d \thetab}{d t} \right\rVert$ (from~\Eqref{eq:norms_equal_loss_descent_main}) concludes the proof.
\end{proof}
\begin{remark}
    Observe that it is the geometric margin induced by the dual norm of the algorithm that is non-decreasing, and not \textit{any} geometric margin. The proof crucially relies on the fact that $\|\nabla \cL (\thetab_t)\|_\star = \left\lVert \frac{d \thetab}{d t} \right\rVert$.
\end{remark}

\subsection{Convergence to KKT Points of the Max-Margin Problem}\label{ssec:main_result}

The previous theorem is a first indication that steepest flow implicitly maximizes the $\lVert \cdot \rVert_\star$-geometric margin in deep neural networks.
However, the monotonicity of the (soft) margin alone does not imply anything about its final value and its optimality.
In this section, we provide a concrete characterization of the {\em asymptotic} behavior of steepest flow: we show that any limit point of the iterates produced by steepest flow is along the direction of a Karush-Kuhn-Tucker (\textit{KKT}) point of the following margin maximization (MM) optimization problem:
\begin{align}\label{eq:min_norm_problem_main}\tag{MM}
    \begin{split}
        &\min_{\thetab \in \mathbb{R}^p} \frac{1}{2}\|\thetab\|^2 \\
        \text{s.t. }&  y_i f(\x_i;\thetab) \geq 1, \; \forall i \in [m].
    \end{split}
\end{align}
Let us recall the definition of a Karush-Kuhn-Tucker point \citep{Kar39,KuTu51}.
\begin{definition}[KKT point]\label{def:kkt}
    A feasible point $\thetab \in \mathbb{R}^p$ of~(\ref{eq:min_norm_problem_main}) is a Karush-Kuhn-Tucker (KKT) point, if there exist $\lambda_1, \ldots, \lambda_m \geq 0$ such that:
    \begin{enumerate}
        \item $\partial \frac{1}{2} \| \thetab \|^2 + \sum_{i = 1}^m \lambda_i \partial \left(1 - y_i f(\x_i;\thetab)\right) \ni 0$.
        \item $\lambda_i (1 - y_i f(\x_i;\thetab)) = 0, \; \forall i \in [m]$.
    \end{enumerate}
\end{definition}
Notice that the first \textit{stationarity} condition is defined using set addition, since we are dealing with non-differentiable functions. See \citet{Dut+13} for more details on optimization problems with non-smooth objectives/constraints. Under some regularity assumptions, the KKT conditions become necessary conditions for global optimality and for non-convex problems like~(\ref{eq:min_norm_problem_main}) they might be the best characterization of optimality we can hope for. See Lemma~\ref{lem:mfcq} in Appendix~\ref{app:proofs} for details.

We are now ready to state our main result.
\begin{theorem}\label{thm:main_result_main}
    Under assumptions (\ref{ass:1}), (\ref{ass:2}), (\ref{ass:3}), consider steepest flow with respect to a norm $\lVert \cdot \rVert$ (\Eqref{eq:steep_flow}) on the exponential loss $\mathcal{L} (\thetab) = \sum_{i = 1}^m e^{- y_i f(\x_i ; \thetab)}$. Then, any limit point $\Bar{\thetab}$ of $\left\{ \frac{\thetab_t}{\|\thetab_t\|}\right\}_{t \geq 0}$ is along the direction of a KKT point of the optimization problem:
    \begin{align}
        \begin{split}
            &\min_{\thetab \in \mathbb{R}^p} \frac{1}{2}\|\thetab\|^2 \\
            \text{s.t. }&  y_i f(\x_i;\thetab) \geq 1, \; \forall i \in [m].
        \end{split}
    \end{align}
\end{theorem}

Theorem~\ref{thm:main_result_main} implies that if the limit of the normalized iterates induced by steepest flow exists, then it is proportional to a KKT point of a margin maximization problem. The proof technique shows that a scaled version of any limit point $\Bar{\thetab}$ is an \textit{approximate} KKT point (Def.~\ref{def:approx_kkt}) with an arbitrarily small error. 
The full proof can be found in Appendix~\ref{app:proofs}.

Note that since Theorem~\ref{thm:main_result_main} is true for \textit{any norm} $\left \lVert \cdot \right \rVert$, the main contribution of \citet{LyLi20}, which characterizes the implicit bias of gradient flow in homogeneous deep networks, can be recovered by our result when $\lVert \cdot \rVert = \lVert \cdot \rVert_2$. Notably, Theorem~\ref{thm:main_result_main} generalizes the result of~\citet{LyLi20} in the following cases of algorithm norms:
\begin{itemize}
    \item Any \textit{$\ell_p$} norm with $p \in [1, \infty]$. This corresponds, for example, to \textit{coordinate descent} (steepest descent with respect to the $\ell_1$ norm) and \textit{sign gradient descent} (normalized steepest descent with respect to the $\ell_\infty$ norm). 
    \item The \textit{modular} norm which was recently introduced by~\citet{Lar+24} to accommodate scalable neural network training. In particular, let a feed-forward neural network with $L+1$ layers be $f\left(\mathbf{x}; \thetab = \{\mathbf{W}_1, \ldots, \mathbf{W}_L, \mathbf{u}\}\right) = \innerprod{\mathbf{u}}{\sigma\left( \mathbf{W}_L \sigma\left( \mathbf{W}_{L-1} \ldots \sigma \left( \mathbf{W}_1 \mathbf{x}\right) \right) \right)}$, where $\sigma: \mathbb{R} \to \mathbb{R}$ is a homogeneous activation applied element-wise. Then, the modular norm induced by the architecture $f$ is given by $\|\thetab\|_{\mathcal{W}} = \max\left( \|\mathbf{W}_L\|_L, \ldots, \|\mathbf{W}_1\|_1, \|\mathbf{u} \|_u \right)$ where each of the norms $\left\lVert \cdot \right\rVert_L, \ldots, \left\lVert \cdot \right\rVert_1, \left \lVert \cdot \right\rVert_u$ can be different. For instance, if the norm of each layer is chosen to be the \textit{spectral} norm, then steepest descent with respect to the modular norm $\left \lVert \cdot \right\rVert_{\mathcal{W}}$ corresponds to Shampoo~\citep{Gup+18} without momentum (see Appendix~\ref{sec:adam} for details).
\end{itemize}

Complementary to Theorem~\ref{thm:main_result_main}, we note that we can obtain a geometric characterization of the evolution of the algorithm after $t_0$: the (generalized)\footnote{See Definition~\ref{def:gen_bregman}.} Bregman divergence induced by the squared norm of the algorithm between the subgradient of the objective function of~\ref{eq:min_norm_problem_main} and the subgradient of its constraints reduces at a rate that depends on the \textit{alignment} between the normalized iterates and the negative normalized, minimum norm, subgradient of the loss $\innerprod{\frac{\thetab_t}{\| \thetab_t\|}}{\frac{- \mathbf{g}_t^\star }{\|\mathbf{g}_t^\star\|_\star}}$. In the case of algorithms whose norm squared is a smooth function, this implies that the finite-time iterates are along the direction of an approximate KKT point of~\ref{eq:min_norm_problem_main}. In contrast, the proof of Theorem~\ref{thm:main_result_main} characterizes the limit points directly and, as a result, does not provide any finite-time characterization of the iterates in terms of proximity to stationarity. See Appendix~\ref{app:bregman_measure} for details.

\paragraph{Comparison to a previous version.} An earlier version of the current work that appeared in ICLR 2025~\citep{TVK25} proved a weaker version of Theorem~\ref{thm:main_result_main}. In particular, it established that, under the same conditions, the limit points of steepest descent are along the direction of a class of \textit{generalized} KKT points of the margin maximization problem that we introduced. The proof proceeded by characterizing the direction of the iterates $\thetab_t$ in terms of their proximity to stationarity. In fact, this weaker notion of stationarity allowed us to explicitly control proximity to (generalized) stationarity, yet it implied convergence to KKT points only in the case of norms whose square is a smooth function. This disqualified interesting and practically relevant cases, such as coordinate descent and sign gradient descent. At a technical level, the proof of Theorem~\ref{thm:main_result_main} does not rely on sequences of approximate KKT points (unlike our previous version and the proof of~\citet{LyLi20}), but instead shows directly that a scaled version of any limit point is stationary.

To the best of our knowledge, this is a first result about the implicit bias of an algorithm in the parameter space of deep homogeneous neural networks which is not about $\ell_2$-geometric margin maximization. 

\section{Experiments}\label{sec:expr}

In this section, we train one-hidden layer neural networks with various steepest descent algorithms (gradient descent-\texttt{GD}, coordinate descent-\texttt{CD}, sign descent-\texttt{SD}) to confirm the validity and measure the robustness of the theoretical claims,
and to discuss the connection between \texttt{Adam}, \texttt{Shampoo} and steepest descent algorithms.
Amongst other quantities, we measure the three relevant geometric margins during training, which, in the context of one-hidden layer neural networks with 1-homogeneous activations and without biases, become:
\begin{equation}\label{eq:margin_def_expr}
        \gamma_1 = \min_{i \in [m]} \frac{y_i f(\x_i;\thetab)}{\|\thetab\|_\infty^2}, \;\;\;
        \gamma_2 = \min_{i \in [m]} \frac{y_i f(\x_i;\thetab)}{\|\thetab\|_2^2}, \;\;\;
        \gamma_\infty = \min_{i \in [m]} \frac{y_i f(\x_i;\thetab)}{\|\thetab\|_1^2}.
\end{equation}

\subsection{Teacher -- Student Experiments}

We first perform experiments in a controlled environment, where the generative process consists of Gaussian data passed through a one-hidden layer (``teacher'') neural network, which is sparse. Specifically:
\begin{equation}\label{eq:gen_process}
    \x \sim \mathcal{N}(0, I_d), \; \; y = \mathrm{sgn}\left(f_{\mathrm{teacher}}\left(\x; \thetab^\star =\{ u_j^\star, \w_j^\star\}_{j = 1}^k\right)\right)= \mathrm{sgn}\left(\sum_{j = 1}^k u_j^\star \sigma\left(\innerprod{\mathbf{w}_j^\star}{\x}\right)\right),
\end{equation}
where $\sigma(u) = \max(u, 0)$ is the ReLU activation, $\mathrm{sgn}(\cdot)$ returns the sign of a number, and 
$\|\thetab^\star\|_0$ is assumed to be small. We train 
 (``student'') neural networks of the same architecture, but of larger width
and with randomly initialized weights: $
f_{\mathrm{student}}(\x; \thetab) = \sum_{j = 1}^{k^\prime} u_j \sigma\left(\innerprod{\mathbf{w}_j}{\x}\right),$
with width $k^\prime > k$ and $w_{jl} \sim \mathcal{U}\left[- \frac{\alpha}{d}, \frac{\alpha}{d}\right], j \in [k^\prime], l \in [d], u_j \in \mathcal{U}\left[- \frac{\alpha}{k^\prime}, \frac{\alpha}{k^\prime} \right]$ (for \texttt{CD} we use: $w_{jl} \sim \mathcal{U}\left[- \frac{\alpha}{k^\prime}, \frac{\alpha}{k^\prime}\right], j \in [k^\prime], l \in [d]$ in order to keep all the individual parameters to the same scale at initialization). The magnitude of initialization $\alpha$ can control how fast the implicit bias of the algorithm kicks in, with smaller values  entering this ``rich'' regime faster \citep{Woo+20}. We compare the performance of (full batch) \texttt{GD}, \texttt{CD} and \texttt{SD} in minimizing the empirical exponential loss on $m$ independent points sampled from the generative process of~\Eqref{eq:gen_process}.
Section~\ref{app:expr_details} contains full experimental details.
\begin{figure}
    \centering
    \includegraphics[scale=0.325]{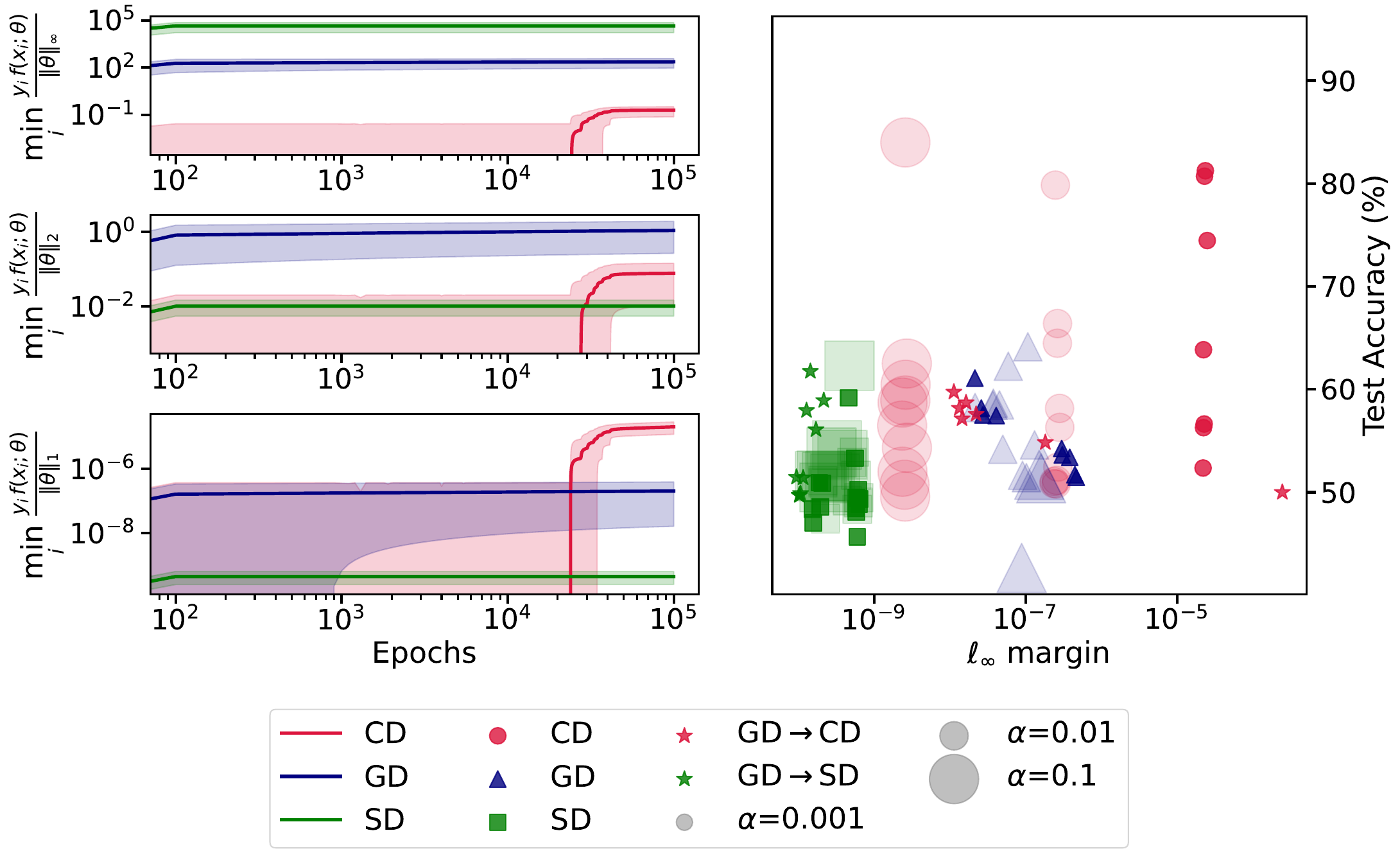}
    \caption{\textbf{Evaluation of steepest descent algorithms in a teacher-student setup.} \textit{Left:} Geometric margins ($\gamma_1, \gamma_2, \gamma_\infty$ in~\Eqref{eq:margin_def_expr}) over the course of training (average over 20 different seeds). \textit{Right:} Final test accuracy vs final $\ell_\infty$ margin ($\gamma_\infty$). Each point in the 2d space corresponds to a different run (only showing runs that did not diverge). Larger points correspond to larger initialization scales $\alpha$. The star points are produced by switching from \texttt{GD} to \texttt{CD} (red) or \texttt{SD} (green), right after the point of perfect train accuracy.}
    \label{fig:combined_student_teacher}
\end{figure}
According to Theorems~\ref{thm:margin_monotonicity_main} and~\ref{thm:main_result_main}, we expect \texttt{GD} to favor solutions with small $\ell_2$ norm. This is equivalent to a small sum of the product of the magnitude of incoming and outcoming weights across all neurons (Theorem 1 in \citet{NTS15}).
On the other hand, \texttt{CD} will seek to minimize the $\ell_1$ norm of the parameters, which translates to a narrow network with sparse 1st-layer weights. Finally, \texttt{SD}'s bias towards small $\|\thetab\|_\infty$ solutions does not appear to be useful for generalizing from few samples in this task. Therefore, we expect \texttt{CD} $>$ \texttt{GD} $>$ \texttt{SD} in terms of generalization.

Figure~\ref{fig:combined_student_teacher} displays our main results. We summarize our key findings below:
\begin{enumerate}
    \item[(i)] \textbf{Margins increase past} $t_0$. As expected from Lemma~\ref{thm:margin_monotonicity_main}, we observe that, right after the point of separation, each algorithm implicitly increases its corresponding geometric margin (Figure~\ref{fig:combined_student_teacher} left). Furthermore, we observe that the ordering of the algorithms is as expected for each margin (\texttt{SD} attains larger $\ell_1$ margin than \texttt{GD} and \texttt{CD}, etc.), despite the fact that Theorem~\ref{thm:main_result_main} only guarantees convergence to a KKT point of the margin maximization problem - note the log-log plot.
    \item[(ii)] \textbf{Smaller initialization produces larger geometric margin}. A smaller magnitude of initialization $\alpha$ causes a larger eventual value of the geometric margin (see Figure~\ref{fig:combined_student_teacher} right for \texttt{CD} and $\gamma_\infty$, where this effect is stronger, and Figure~\ref{fig:add_margin} in Appendix~\ref{app:expr_details} for $\gamma_1, \gamma_2$). 
    \item[(iii)] \textbf{Importance of early-stage dynamics for generalization}. 
    Figure~\ref{fig:combined_student_teacher}, right, shows the final test accuracy of the networks (20 different runs) vs the value of their final $\ell_\infty$ margin ($\gamma_\infty$). We observe that, while there exist more \texttt{CD} runs with good generalization (red circles), these do not always coincide with larger $\gamma_\infty$. Furthermore, intervening in the algorithms to encourage or discourage $\gamma_\infty$-maximization does not result in significant generalization changes: after running \texttt{GD} until the point of perfect train accuracy, we switch to either \texttt{SD} (green stars) or \texttt{CD} (red stars) to directly control the late stage geometric properties of the model. Switching to \texttt{CD} seems to bear marginal benefits in terms of generalization, even though all the switched runs reach smaller values of $\ell_\infty$ margin compared to the full, no-switching, \texttt{GD} runs. These benefits, however, pale in comparison to the full \texttt{CD} runs. Switching to \texttt{SD}, on the other hand, results in smaller $\gamma_\infty$ and similar or marginally worse test accuracy. See also Figure~\ref{fig:add_margin} in Appendix~\ref{app:expr_details} for test accuracy vs the other two geometric margins. We conclude that it is unlikely that large generalization benefits can solely and causally be linked to larger geometric margins in this setup, and it appears that the early stage dynamics play an important role for generalization.
\end{enumerate}

\subsection{Connection between Adam and Sign-GD}
\begin{wrapfigure}{r}{0.415\textwidth}
    \vspace{-7.5mm}
    \begin{center}
        \includegraphics[scale=0.195]{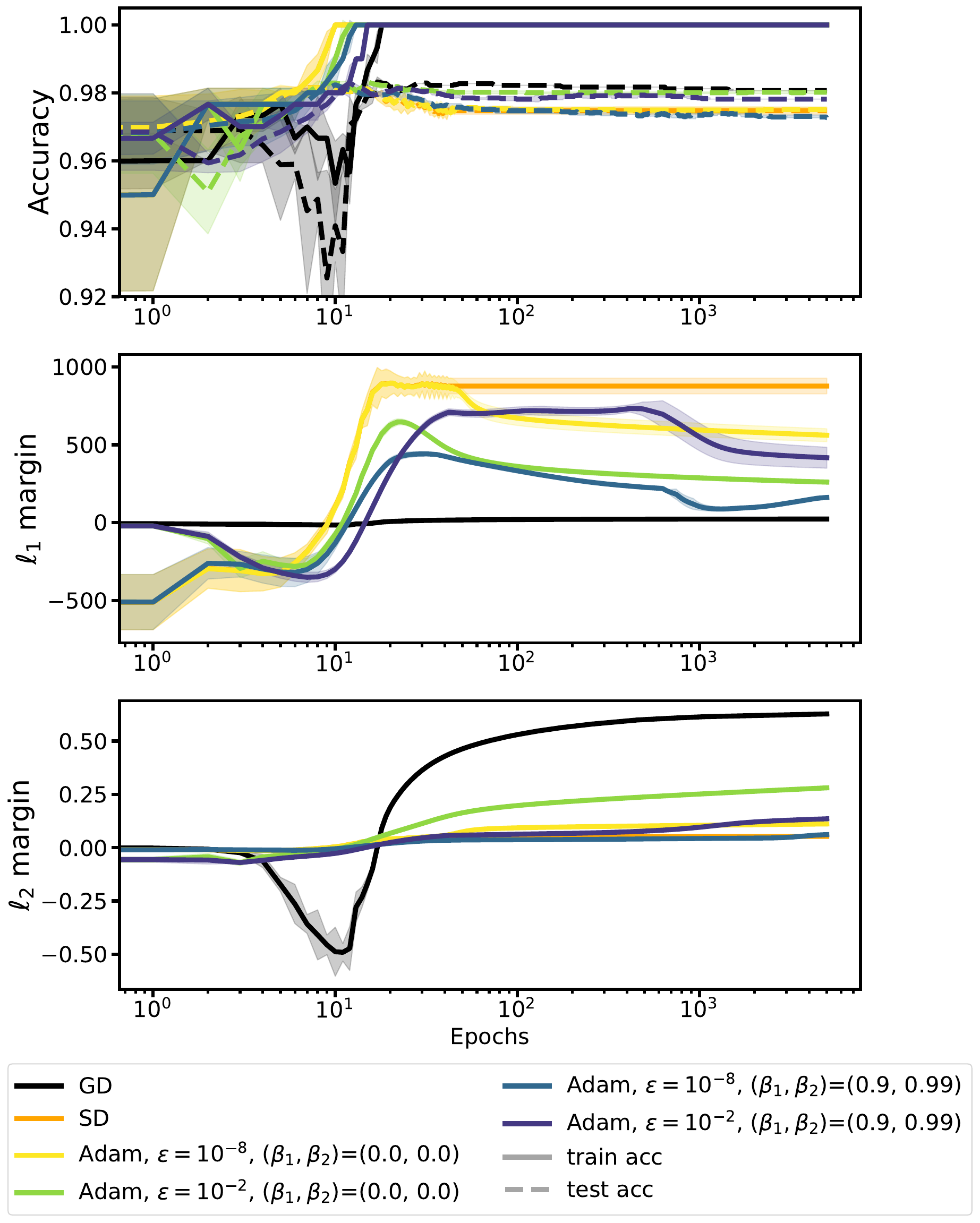}
    \end{center}
    \caption{\textbf{Accuracy, and $\ell_1, \ell_2$ margins  during training for \texttt{GD}, \texttt{SD} and \texttt{Adam} on MNIST (3 random seeds)}. Adam is parameterized by a numerical precision constant $\epsilon$ and two momentum parameters $(\beta_1, \beta_2)$ (defaulting to $10^{-8}$ and $(0.9, 0.99)$). We observe that \texttt{Adam} behaves similarly to \texttt{SD} for the period right after the point of perfect train accuracy.}
    \label{fig:mnist_combined}
\end{wrapfigure}
Adaptive optimization methods like \texttt{Adam} \citep{KiBa15} have been popular in deep learning applications, yet theoretically their value has been questioned \citep{Wil+17} and their properties remain poorly understood. \citet{Wan+21,Wan+22}  studied the implicit bias of \texttt{Adam} in homogeneous networks and concluded that \texttt{Adam} shares the same asymptotic properties as \texttt{GD}. More recently, this conclusion has been challenged~\citep{ZZC24}, in the sense that this asymptotic property crucially depends on a precision parameter of the algorithm and does not capture realistic runs of the algorithm (see Section~\ref{sec:adam} for details). In particular, it was shown that in linear models, \texttt{Adam}, without this precision parameter, implicitly maximizes the $\ell_1$-geometric margin \citep{ZZC24}, a property shared with \texttt{SD} and not \texttt{GD}. Indeed, \texttt{Adam} without momentum, and ignoring the precision parameter, is equivalent to \texttt{SD} (see Section~\ref{sec:adam}).
Setting the precision parameter to 0, on the other hand, is not useful in applications, as small initial values of the gradient result in divergence of the loss. A question arises: what, then, are the relevant geometric properties of \texttt{Adam} \textit{in practice}?

Figure~\ref{fig:mnist_combined} provides some experimental answers to this question, in light of Theorems~\ref{thm:margin_monotonicity_main} and~\ref{thm:main_result_main}. On a pair of digits extracted from MNIST we train two-layer neural networks  with \texttt{GD}, \texttt{SD} and \texttt{Adam}, with small initialization. See Section~\ref{app:expr_details} for experimental details. We observe that, as soon as the algorithms reach 100$\%$ train accuracy, the margins start to increase (as Theorem~\ref{thm:margin_monotonicity_main} suggests); \texttt{SD} reaches a larger value of $\gamma_1$, while \texttt{GD} reaches a larger value of $\gamma_2$. Interestingly, \texttt{Adam} with the default hyperparameters (precision $\epsilon=10^{-8}$ and non-zero momentum), initially, behaves similarly to \texttt{SD}, increasing $\gamma_1$, before it starts decreasing it, in order to slowly start increasing $\gamma_2$! Curiously, larger values of $\epsilon$ increase $\gamma_1$ even further and start the second phase slower, but more aggressively. Notice, however, that train and test accuracies have long converged, so it is unlikely that a typical run would have lasted long enough to see the second phase of $\ell_2$-margin maximization (in particular, the loss value needs to be smaller than $10^{-7}$ in order to observe such behavior). Similar observations hold for \texttt{Adam} without momentum (recall that without momentum and for $\epsilon \to 0$, we recover \texttt{SD}). Therefore, it appears that the $\ell_1$ bias of \texttt{SD} (Theorems~\ref{thm:margin_monotonicity_main},~\ref{thm:main_result_main} for $\lVert \cdot \rVert = \lVert \cdot \rVert_\infty$)  more faithfully describes a typical run of \texttt{Adam} in neural networks.

\subsection{Shampoo favors a spectral margin}

\begin{figure}
    \centering
    \includegraphics[scale=0.325]{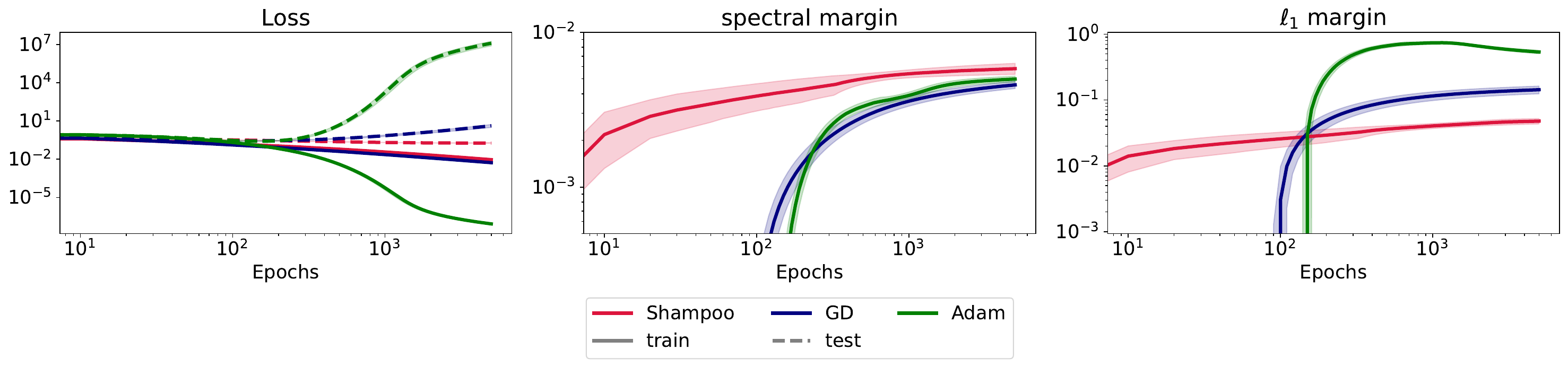}
    \caption{\textbf{Loss, spectral and $\ell_1$ margin during training for \texttt{GD}, \texttt{Adam} and \texttt{Shampoo} on MNIST with a two-layer neural network with a frozen second layer}. We observe that \texttt{Shampoo} ends up with a larger spectral margin.}
    \label{fig:shampoo}
\end{figure}

\citet{Gup+18} introduced \texttt{Shampoo} as an efficient method for deep learning optimization and the algorithm has attracted a lot of interest from practitioners and theorists alike. At a high level, \texttt{Shampoo} is a preconditioning method which respects the structure of the optimization space on which it operates. The algorithm and its variants have shown to speed up convergence in a variety of settings, including language modeling with autoregressive transformers~\citep{Jor+24}. However, little is understood about the generalization properties of the algorithm. Our results let us clarify some of its late-stage properties.

In particular, as observed by~\citet{BeNe24}, \texttt{Shampoo} operating on each layer with accumulation turned off corresponds to steepest descent with respect to the spectral norm of the weight matrix of that layer. Overall, it corresponds to steepest descent with respect to an architecture-dependent norm in parameter space. See Appendix~\ref{ssec:shampoo} for more details. For a two-layer neural network with a fully-connected layer, a frozen second layer and scalar output $f\left(\mathbf{x}; \mathbf{W}\right) = \innerprod{\mathbf{u}}{\sigma\left(\mathbf{W}\mathbf{x}\right)}$, \texttt{Shampoo} is steepest descent with respect to the spectral norm $\sigma_{\max}(\mathbf{W})$. As a result, Theorems~\ref{thm:margin_monotonicity_main} and \ref{thm:main_result_main} predict that \texttt{Shampoo} implicitly maximizes the \textit{spectral} margin:
\begin{equation}
    \gamma_\sigma = \min_{i \in [m]} \frac{y_i f(\x_i;\mathbf{W})}{ \sigma_{\max}(\mathbf{W})}.
\end{equation}
In Figure~\ref{fig:shampoo} we provide experimental evidence for the above. We train two-layer neural networks (with a frozen second layer) on a pair of digits extracted from MNIST, starting from small initialization. We compare an off-the-shelf implementation of \texttt{Shampoo}~\citep{Nov20} to \texttt{Adam} and \texttt{GD}, all without momentum. As we observe, the spectral margin $\gamma_\sigma$ increases once the algorithms reach perfect train accuracy, while \texttt{Shampoo} achieves the largest value among the three.

\section{Conclusion}

In our work, we considered the large family of steepest descent algorithms with respect to an arbitrary norm $\lVert \cdot \rVert$ and provided a unifying theoretical analysis of their late-stage implicit bias when training homogeneous neural networks. Our main result establishes that homogeneous neural networks implicitly maximize the geometric margin of the training points induced by the norm of the algorithm.

Our results can reinforce several recent efforts that attempt to understand deep learning through the lens of implicit bias. In particular, questions about generalization, robustness, and privacy can now be asked more broadly: (a) What are the important properties of language data distributions which render algorithms such as Adam ($\approx$ sign gradient descent) and Shampoo ($\approx$ spectral steepest descent) appealing for training deep networks on such tasks?
(b) An interesting connection has been made between implicit bias and the ability to retrieve data points from the parameters of a trained neural network~\citep{Hai+22}. In particular, when the parameters converge to a stationary point of a certain margin-maximization problem, they satisfy equations with respect to the training dataset, which allow to reconstruct portions of the training data.  Our work leads to the question:  Is it also possible to extract training samples from neural networks optimized with Adam or Shampoo?(c) Can we leverage our implicit bias results to design more sample-efficient algorithms for adversarially robust training, as argued by \citet{Tsi+24}? (d) A central topic in deep network generalization is their ability to avoid overfitting despite the presence of noise in the data. In prior work, this phenomenon has been attributed to the geometric properties of the networks, which are a consequence of the implicit bias of gradient descent~\citep{Fre+22,Sha23}. Is benign overfitting an inherent feature of first-order methods, or are current results specifically tailored to gradient descent?

\paragraph{Acknowledgments.}
NT and JK acknowledge support through the NSF under award 1922658. 
GV is supported by The Israel Science Foundation (grant No. 2574/25), by a research grant from Mortimer Zuckerman (the Zuckerman STEM Leadership Program), and by research grants from the Center for New Scientists at the Weizmann Institute of Science, and the Shimon and Golde Picker -- Weizmann Annual Grant.
We would like to thank anonymous referees that helped improving an earlier version of this work. We would like to thank Zhiyuan Li and Kaifeng Lyu for a useful discussion. Early parts of this work were done while NT was visiting the Toyota Technological Institute of Chicago (TTIC) during the winter of 2024, and NT would like to thank everyone at TTIC for their hospitality, which enabled this work. Part of this work was done while JK and NT were hosted by the Centre Sciences de Donnees at the
École Normale Supérieure (ENS) in 2023/24, whose hospitality we gratefully acknowledge. This work was supported in part through the NYU IT High Performance Computing resources, services, and staff expertise.

\bibliography{refs}
\bibliographystyle{apalike}

\newpage

\appendix


\section{Proof Details}\label{app:proofs}
In this section, we provide proofs for the results stated in the main text.

\subsection{Steepest Flow}

We first present a series of technical results, which are about steepest flow in the case of non-differentiable loss functions.
In what follows, we will denote with $\mathbf{g}_t^\star$ any loss subderivative with minimum $\left\lVert \cdot \right\rVert_\star$ norm, i.e. $\mathbf{g}_t^\star \in \argmin_{\mathbf{u} \in \partial \cL(\thetab_t)} \|\mathbf{u}\|_\star$.
In the case of subdifferentials, chain rule holds as an inclusion:
\begin{theorem}[Theorem 2.3.9 and 2.3.10 in \cite{Cla90}]\label{thm:chain_rule_inclusion}
    Let $z_1, \ldots, z_n : \mathbb{R}^d \to \mathbb{R}$, $f: \mathbb{R}^n \to \mathbb{R}$ be locally Lipschitz functions and define $\mathbf{z} = (z_1, \ldots, z_n)$. Let $(f \circ \mathbf{z}) (\mathbf{x}) = f(z_1(\mathbf{x}), \ldots, z_n(\x))$ be the composition of $\mathbf{z}$ with $f$. Then, it holds:
    \begin{equation}
        \partial (f \circ \mathbf{z}) (\mathbf{x}) \subseteq \mathrm{conv}\left\{ \sum_{i = 1}^n \alpha_i \mathbf{h}_i : \bm{\alpha} \in \partial f(z_1(\x), \ldots, z_n(\x)), \mathbf{h}_i \in \partial z_i(\x) \right\}.
    \end{equation}
\end{theorem}
To further analyze steepest flows and to guarantee loss monotonicity, we need a stronger chain rule result.
This can be achieved for a large class of locally Lipschitz functions, as per the following theorem which is due to \cite{Dav+20}. 
\begin{theorem}{(Theorem 5.8 in \cite{Dav+20})}\label{thm:chain_rule}
    If $F: \mathbb{R}^k \to \mathbb{R}$ is locally Lipschitz and Whitney $C^1$-stratifiable, then it admits a chain rule: for all arcs (functions which are absolutely continuous on every compact subinterval) $\mathbf{u}: [0, \infty) \to \mathbb{R}^k$, almost all $t \geq 0$, and all $\mathbf{g} \in \partial F(\mathbf{u}(t))$, it holds:
    \begin{equation}
        \frac{d F(\mathbf{u}(t))}{dt} = \innerprod{\mathbf{g}}{\frac{d \mathbf{u}(t)}{dt}}.
    \end{equation}
\end{theorem}
Whitney $C^1$-stratifiability includes a large family of functions, including functions defined in an o-minimal structure which has been a previous assumption in the literature \citep{JiTe20}. It excludes some pathological functions - see, for instance, Appendix J in \citep{LyLi20}. This version of chain rule allows us to derive the following central properties of steepest flows.
\begin{proposition}\label{prop:steep_flow_duality}
    Let $\cL: \mathbb{R}^p \to \mathbb{R}$ and assume that $\cL$ admits a chain rule. Then, for the steepest flow iterates of~\Eqref{eq:steep_desc}, it holds for almost any $t \geq 0$:
    \begin{equation}\label{eq:loss_descent}
        \frac{d \cL}{d t} = - \left\lVert \frac{d \thetab}{d t} \right\rVert^2 \leq 0,
    \end{equation}
    and
    \begin{equation}\label{eq:norms_equal}
        \innerprod{\frac{d \thetab}{d t}}{- \mathbf{g}_t^\star} = \left\lVert \frac{d \thetab}{d t} \right\rVert^2 = \|\mathbf{g}_t^\star\|_\star^2,
    \end{equation}
    where $\mathbf{g}_t^\star \in \argmin_{\mathbf{u} \in \partial \cL(\thetab_t)} \|\mathbf{u}\|_\star$.
\end{proposition}
\begin{proof}
    From Theorem~\ref{thm:chain_rule}, for almost any $t \geq 0$, it holds $\forall \; \mathbf{g}_t \in \partial \cL(\thetab_t)$:
    \begin{equation}\label{eq:chain_rule_general}
        \frac{d \cL}{d t} = \innerprod{\mathbf{g}_t}{\frac{d \thetab}{d t}}.
    \end{equation}
    Applying this for the element of $\partial \cL(\thetab_t), \mathbf{g}_t^\prime$, that corresponds to $\frac{d \thetab}{dt}$ from the definition of steepest flow~\Eqref{eq:steep_flow}, we get:
    \begin{equation}\label{eq:one_min_norm_path}
        \frac{d \cL}{d t} = \innerprod{\mathbf{g}_t^\prime}{\frac{d \thetab}{d t}} = - \left\lVert \frac{d \thetab}{d t} \right\rVert^2,
    \end{equation}
    where the last equality follows from the definition of the dual norm. But, \Eqref{eq:chain_rule_general} for $\mathbf{g}_t^\star \in \argmin_{\mathbf{u} \in \partial \cL(\thetab_t)} \|\mathbf{u}\|_\star$, yields:
    \begin{equation}\label{eq:two_min_norm_path}
        \left \lvert \frac{d \cL}{d t} \right \rvert = \left \lvert \innerprod{\mathbf{g}_t^\star}{\frac{d \thetab}{d t}} \right \rvert \leq \|\mathbf{g}_t^\star\|_\star \left\lVert \frac{d \thetab}{d t} \right\rVert.
    \end{equation}
    Thus, combining \Eqref{eq:one_min_norm_path},~\Eqref{eq:two_min_norm_path}, we obtain:
    \begin{equation}
        \left\lVert \frac{d \thetab}{d t} \right\rVert \leq \| \mathbf{g}_t^\star\|_\star,
    \end{equation}
    which implies that the update rule~\Eqref{eq:steep_flow} is equivalent to:
    \begin{equation}
        \frac{d \thetab}{d t} \in \left\{ \argmin_{\| \mathbf{u} \| \leq \|\mathbf{g}_t^\star\|_\star} \innerprod{\mathbf{u}}{\mathbf{g}_t^\star}: \mathbf{g}_t^\star \in \argmin_{\mathbf{u} \in \partial \cL(\thetab_t)} \|\mathbf{u}\|_\star \right\}.
    \end{equation}
    Therefore, from the definition of the dual norm, we have:
    \begin{equation}
        \innerprod{\frac{d \thetab}{d t}}{- \mathbf{g}_t^\star} = \left\lVert \frac{d \thetab}{d t} \right\rVert^2 = \|\mathbf{g}_t^\star\|_\star^2.
    \end{equation}
\end{proof}
Hence, under the mild assumptions of Theorem~\ref{thm:chain_rule}, the loss is non-increasing during training. 

\subsection{Late Phase Implicit Bias}\label{ssec:late_phase_proofs}

A useful standard characterization of the subdifferential of a norm is the following:
\begin{lemma}\label{lem:subgrad_norm}
    $$ \partial \|\mathbf{x}\| = \{\mathbf{v}: \innerprod{\mathbf{v}}{\mathbf{x}} = \|\mathbf{x}\|, \|\mathbf{v}\|_\star \leq 1\} $$
\end{lemma}

We present the proofs for our results about the late stage of training in steepest flow algorithms. The next lemma quantifies the behavior of the smooth margin past the point $t_0$ (where, recall, zero classification error is achieved).
\begin{theorem}[Soft margin increases - full version]\label{thm:margin_monotonicity}
For almost any $t > t_0$, it holds: 
$$\frac{d \log \Tilde{\gamma}}{d t} \geq L \left\lVert \frac{d \thetab}{d t} \right\rVert^2 \left( \frac{1}{L \cL(\thetab_t) \log\frac{1}{\cL(\thetab_t)}} - \frac{1}{\|\thetab_t\|\left\lVert \frac{d \thetab}{d t} \right\rVert}\right) \geq 0.$$
\end{theorem}
\begin{proof}
    Let $\mathbf{n}_t \in \partial \|\thetab_t\|$. We have:
    \begin{equation}\label{eq:margin_lemma_eq1}
        \begin{split}
            \frac{d \log \Tilde{\gamma}}{d t} & = \frac{d}{d t} \log \log \frac{1}{\cL (\thetab_t)} - L \frac{d}{d t} \log \|\thetab_t\| \\
            & = \frac{d}{d t} \log \log \frac{1}{\cL (\thetab_t)} - L \innerprod{\frac{\mathbf{n}_t}{\|\thetab_t\|}}{\frac{d \thetab}{d t}} \;\;\;\;\; (\text{Chain rule})  \\
            & \geq \frac{d}{d t} \log \log \frac{1}{\cL (\thetab_t)} - L \frac{\left\lVert \frac{d \thetab}{d t} \right\rVert}{\|\thetab_t\|} \;\;\;\;\; (\text{definition of dual norm and $\|\mathbf{n}_t \|_\star \leq 1$}) \\
            & = - \frac{d \cL (\thetab_t)}{d t} \frac{1}{ \cL(\thetab_t) \log\frac{1}{\cL(\thetab_t)}} - L \frac{\left\lVert \frac{d \thetab}{d t} \right\rVert}{\|\thetab_t\|} \;\;\;\;\; (\text{Chain rule}) \\
            & = \left\lVert \frac{d \thetab}{d t} \right\rVert^2 \left( \frac{1}{ \cL(\thetab_t) \log\frac{1}{\cL(\thetab_t)}} - \frac{L}{\|\thetab_t\|\left\lVert \frac{d \thetab}{d t} \right\rVert}\right) \;\;\;\;\; (\text{\Eqref{eq:loss_descent}}).
        \end{split}
    \end{equation}
    But, the first term inside the parenthesis can be related to the second one via the following calculation. Recall that, by Theorem~\ref{thm:chain_rule}, for any $\mathbf{g}_t \in \partial \cL (\thetab_t)$ there exist $\mathbf{h}_1 \in \partial y_1 f(\x_1; \thetab_t), \ldots, \mathbf{h}_m \in \partial y_m f(\x_m; \thetab_t)$ such that $\mathbf{g}_t = \sum_{i = 1}^m e^{- y_i f(\x_i ; \thetab_t)} \mathbf{h}_i$. Thus, for a minimum norm subderivative $\mathbf{g}_t^\star$, we have:
    \begin{equation}
        \begin{split}
            \innerprod{\thetab_t}{- \mathbf{g}_t^\star} & = \innerprod{\thetab_t}{\sum_{i = 1}^m e^{- y_i f(\x_i ; \thetab_t)} \mathbf{h}_i^\star} \\
            & = \sum_{i = 1}^m e^{- y_i f(\x_i ; \thetab_t)} \innerprod{\thetab_t}{\mathbf{h}_i^\star} \\
            & = L \sum_{i = 1}^m  e^{- y_i f(\x_i ; \thetab_t)} y_i f(\x_i ; \thetab_t),
        \end{split}
    \end{equation}
    where the last equality follows from Euler's theorem for homogeneous functions (whose generalization for subderivatives can be found in Theorem B.2 in \cite{LyLi20}). Now, observe that this last term can be lower bounded as:
    \begin{equation}\label{eq:homog_ineq}
        \begin{split}
            \innerprod{\thetab_t}{- \mathbf{g}_t^\star} \geq L \sum_{i = 1}^m e^{- y_i f(\x_i ; \thetab_t)} \min_{i \in [m]} y_i f(\x_i ; \thetab_t) \geq L \cL (\thetab_t) \log \frac{1}{\cL (\thetab_t)},
        \end{split}
    \end{equation}
    where we used the fact $e^{- \min_{i \in [m]} y_i f(\x_i ; \thetab_t)} \leq \sum_{i = 1}^m e^{-y_i f(\x_i ; \thetab_t)}$. We have made the first term of~\Eqref{eq:margin_lemma_eq1} appear. By plugging~\Eqref{eq:homog_ineq} into~\Eqref{eq:margin_lemma_eq1}, we get:
    \begin{equation}
        \begin{split}
            \frac{d \log \Tilde{\gamma}}{d t} & \geq \left\lVert \frac{d \thetab}{d t} \right\rVert^2 \left( \frac{L}{\innerprod{\thetab_t}{-\mathbf{g}_t^\star}} - \frac{L}{\|\thetab_t\|\left\lVert \frac{d \thetab}{d t} \right\rVert}\right) \\
            & \geq \left\lVert \frac{d \thetab}{d t} \right\rVert^2 \left( \frac{L}{\|\thetab_t\| \|\mathbf{g}_t^\star\|_\star} - \frac{L}{\|\thetab_t\|\left\lVert \frac{d \thetab}{d t} \right\rVert}\right) \;\;\;\;\; (\text{definition of dual norm}).
        \end{split}
    \end{equation}
    Noticing that $\|\mathbf{g}_t^\star\|_\star = \left\lVert \frac{d \thetab}{d t} \right\rVert$ (from Proposition~\ref{prop:steep_flow_duality}) concludes the proof.
\end{proof}

By extending Lemma B.6 of \cite{LyLi20}, we can further prove that the loss converges to 0 and, thus, the norm of the iterates diverges to infinity.

\begin{lemma}\label{lem:loss_and_norm_at_inf}
    As $t\to \infty$, $\mathcal{L}(\thetab_t) \to 0$ and $\|\thetab_t\| \to \infty$.
\end{lemma}
\begin{proof}
    We suppress the dependence of the loss and the iterates from time $t$, when it is obvious from the context.
    
    From the definition of the steepest flow update and chain rule (\Eqref{eq:loss_descent}), we have
    \begin{equation}\label{eq:norm_grad_dual_ineq}
        - \frac{d \cL}{d t} = \left\lVert \frac{d \thetab}{d t} \right\rVert^2 = \| \mathbf{g}_t^\star \|_\star^2 \geq \frac{1}{\|\thetab\|^2} \innerprod{\thetab}{- \mathbf{g}_t^\star}^2,
    \end{equation}
    where we applied~\Eqref{eq:norms_equal},~\Eqref{eq:loss_descent} and the definition of the dual norm. But, as we showed in~\Eqref{eq:homog_ineq}, the above inner product can be upper bounded by a function of the loss, so, by plugging in, we get:
    \begin{equation}
        -\frac{d \cL}{d t} \geq \frac{L^2}{\|\thetab\|^2} \left( \cL \log\frac{1}{\cL} \right)^2 = \frac{L^2}{\left( \log \frac{1}{\cL} \right)^{2/L}} \Tilde{\gamma}^{2/L} (t) \left( \cL \log \frac{1}{\cL}\right)^2 \geq \frac{L^2}{\left( \log \frac{1}{\cL} \right)^{2/L}} \Tilde{\gamma}^{2/L} (t_0) \left( \cL \log \frac{1}{\cL}\right)^2,
    \end{equation}
    which follows from the definition of the margin~(\Eqref{eq:smooth_margin}) and its monotonicity (Lemma \ref{thm:margin_monotonicity}).
    By rearranging:
    \begin{equation}
        -\frac{d \cL}{d t} \frac{1}{\cL^2} \left(\log \frac{1}{\cL}\right)^{2/L - 2} \geq L^2 \Tilde{\gamma}(t_0)^{2/L},
    \end{equation}
    and integrating over time from $t_0$ to $t > t_0$, we further have:
    \begin{equation}
        \int_{t_0}^t \left(\log \frac{1}{\cL}\right)^{2/L - 2} \frac{d}{d t} \frac{1}{\cL} dt \geq L^2 \Tilde{\gamma}(t_0)^{2/L} (t-t_0),
    \end{equation}
    or, by a change of variables,
    \begin{equation}
        \int_{1/\cL(t_0)}^{1/\cL({t})} \left(\log u\right)^{2/L - 2} du \geq L^2 \Tilde{\gamma}(t_0)^{2/L} (t-t_0).
    \end{equation}
    The RHS diverges to infinity as $t\to \infty$, hence so does the LHS, which can only happen if $\cL \to 0$. In order for $\cL(\thetab_t) = \sum_{i = 1}^m e^{-y_i f(\x_i;\thetab_t)} = \sum_{i = 1}^m e^{-y_i \|\thetab_t\|^L f\left(\x_i;\frac{\thetab_t}{\|\thetab_t\|}\right)}$ to go to zero, it must be $\|\thetab_t\| \to \infty$.
\end{proof}

The following Lemma quantifies the connection between soft and hard margin.

\begin{lemma}\label{lem:smooth_hard_margin}
    For any $\thetab$, it holds:
    \begin{equation}\label{eq:soft_hard_ineq}
        \frac{\min_{i \in [m]} y_i f(\x_i; \thetab) - \log m}{\|\thetab\|^L} \leq \Tilde{\gamma} \leq \frac{\min_{i \in [m]} y_i f(\x_i; \thetab)}{\|\thetab\|^L}.
    \end{equation}
\end{lemma}
\begin{proof}
    Follows from:
    \begin{equation}
        e^{- \min_{i \in [m]} y_i f(\x_i; \thetab)} \leq \cL (\thetab) \leq m e^{- \min_{i \in [m]} y_i f(\x_i; \thetab)}.
    \end{equation}
\end{proof}

From the previous two Lemmata, we deduce that the soft margin converges to the hard margin as $t \to \infty$.
\begin{corollary}\label{cor:soft_hard_convergence}
For any $t > t_0$, $\thetab_t \in \mathbb{R}^p$, let $\gamma(\thetab_t) = \frac{\min_{i \in [m]} y_i f(\x_i; \thetab_t)}{\|\thetab_t\|^L}$. Then, it holds:
    \begin{equation}
        \lim_{t\to \infty} \vert \Tilde{\gamma}(\thetab_t) - \gamma(\thetab_t) \vert = 0.
    \end{equation}
\end{corollary}
\begin{proof}
    By taking limits in~\Eqref{eq:soft_hard_ineq}, we have:
    \begin{equation}
        \begin{split}
            \lim_{t\to \infty} \frac{\min_{i \in [m]} y_i f(\x_i; \thetab)}{\|\thetab\|^L} - \frac{\log m}{\|\thetab\|^L} \leq \lim_{t\to\infty} \Tilde{\gamma}(\thetab_t) & \leq \lim_{t\to\infty} \frac{\min_{i \in [m]} y_i f(\x_i; \thetab)}{\|\thetab\|^L} \iff \\
            \lim_{t\to \infty} \gamma(\thetab_t) - \lim_{t \to \infty} \frac{\log m}{\|\thetab\|^L} \leq \lim_{t\to\infty} \Tilde{\gamma}(\thetab_t) & \leq \lim_{t\to\infty} \gamma(\thetab_t).
        \end{split}
    \end{equation}
    But, from Lemma~\ref{lem:loss_and_norm_at_inf}, we know that $\|\thetab_t\| \to \infty$. Thus, 
    \begin{equation}
        \lim_{t\to \infty} \gamma(\thetab_t) \leq \lim_{t\to\infty} \Tilde{\gamma}(\thetab_t) \leq \lim_{t\to\infty} \gamma(\thetab_t),
    \end{equation}
    which shows the claim.
\end{proof}

The last part of the proof consists of characterizing the (directional) convergence of the iterates in relation to stationary points of the following optimization problem (re-introduced here for convenience):
\begin{align}\label{eq:min_norm_problem_2}
    \begin{split}
        &\min_{\thetab \in \mathbb{R}^p} \frac{1}{2}\|\thetab\|^2 \\
        \text{s.t. }&  y_i f(\x_i;\thetab) \geq 1, \; \forall i \in [m].
    \end{split}
\end{align}

We first show that we can always construct a feasible point of~\Eqref{eq:min_norm_problem_2} from a scaled version of $\thetab_t$.

\begin{lemma}\label{lem:feasible}
    For any $t > 0$, $\Tilde{\thetab}_t = \frac{\thetab_t}{\left(\min_{i \in [m]} y_i f(\x_i; \thetab_t)\right)^{\frac{1}{L}}}$ is a feasible point of~\Eqref{eq:min_norm_problem_2}.
\end{lemma}
\begin{proof}
    From the homogeneity of $f$, we have:
    \begin{equation}
        \begin{split}
            y_i f(\x_i; \Tilde{\thetab_t}) & = y_i f\left(\x_i; \frac{\thetab_t}{\left(\min_{i \in [m]} y_i f(\x_i; \thetab_t)\right)^{\frac{1}{L}}}\right) = \frac{y_i f(\x_i; \thetab_t)}{\min_{i \in [m]} y_i f(\x_i; \thetab_t)} \geq 1
        \end{split}
    \end{equation}
    for all $i \in [m]$. So $\Tilde{\thetab}_t$ is a feasible point of~\Eqref{eq:min_norm_problem_2}.
\end{proof}

Under some regularity assumptions, the KKT conditions (Definition~\ref{def:kkt}) become necessary for global optimality (yet, not sufficient):
\begin{definition}
    We say that a feasible point of~\Eqref{eq:min_norm_problem_2} satisfies the \textit{Mangasarian-Fromovitz Constraint Qualifications} if there exists $\mathbf{v} \in \mathbb{R}^p$ such that for all $i \in [m]$ with $1 - y_i f(\x_i;\thetab) = 0$ and for all $\mathbf{h} \in \partial  \left(1 - y_i f(\x_i; \thetab) \right)$, it holds:
    \begin{equation}
        \innerprod{\mathbf{v}}{\mathbf{h}} > 0.
    \end{equation}
\end{definition}

The next Lemma shows that Problem~\ref{eq:min_norm_problem_2} satisfies the Mangasarian-Fromovitz Constraint Qualifications:
\begin{lemma}\label{lem:mfcq}
    Problem~\ref{eq:min_norm_problem_2} satisfies the Mangasarian-Fromovitz Constraint Qualifications at every feasible point $\thetab$.
\end{lemma}
\begin{proof}
    Let $\mathbf{h}_i \in \partial (1 - y_i f(\x_i; \thetab))$ and $\mathbf{v} = - \thetab$, then for all $i \in [m]$ satisfying $y_i f(\x_i; \thetab) = 1$,  we have from Euler's theorem for homogeneous functions:
    \begin{equation}
        \innerprod{\mathbf{v}}{\mathbf{h}_i} = L y_i f(\x_i; \thetab) = L > 0.
    \end{equation}
\end{proof}

Our proof technique depends on the following relaxed notion of stationarity.
\begin{definition}[$(\epsilon, \delta)$-approximate KKT point]\label{def:approx_kkt}
    A feasible point $\thetab$ of Problem~\eqref{eq:min_norm_problem_2} is called an $( \epsilon, \delta)-$\textit{approximate} KKT point if there exist $\lambda_1, \ldots, \lambda_m \geq 0$, $\mathbf{h}_i \in \partial f(\x_i;\thetab)$ and $\mathbf{k} \in \partial \frac{1}{2} \| \thetab\|^2$ such that:
    \begin{enumerate}
        \item $\left\lVert \sum_{i = 1}^m \lambda_i y_i \mathbf{h}_i - \mathbf{k} \right\rVert_2 \leq \epsilon.$\label{eq:dedelta_condition_one}
        \item $\sum_{i = 1}^m \lambda_i (y_i f(\x_i;\thetab) - 1) \leq \delta$.
    \end{enumerate}
\end{definition}

Before we proceed with the main result, we state and prove three useful Lemmata. The first one lower bounds the alignment between normalized iterates and normalized loss gradients. This Lemma is key for showing that the alignment goes to 1 as $t \to \infty$.

\begin{lemma}\label{lem:alignment_bound}
    For all $t_2 > t_1 \geq t_0$, there exists $t_\star \in (t_1, t_2)$ such that:
    \begin{equation}
        \left( \frac{1}{\innerprod{\frac{\thetab_{t_\star}}{\|\thetab_{t_\star}\|}}{\frac{-\mathbf{g}_{t_\star}^\star}{\left \lVert \mathbf{g}_{t_\star}^\star \right \rVert_\star}}} - 1 \right) \leq \frac{1}{L} \frac{\log \frac{\Tilde{\gamma}(t_2)}{\Tilde{\gamma}(t_1)}}{\displaystyle\int_{t_1}^{t_2} \frac{\left\lVert \frac{d \thetab_t}{d t} \right\rVert}{\|\thetab_t\|}dt},
    \end{equation}
    for all $\mathbf{g}_{t_\star}^\star \in \argmin_{\mathbf{u} \in \partial \cL(\thetab_{t_\star})} \|\mathbf{u}\|_\star$
\end{lemma}
\begin{proof}
    From Lemma \ref{thm:margin_monotonicity}, we have for all $\mathbf{g}_t^\star \in \argmin_{\mathbf{u} \in \partial \cL(\thetab_t)} \|\mathbf{u}\|_\star$:
    \begin{equation}\label{eq:alignment_intermediate}
        \begin{split}
            \frac{d \log \Tilde{\gamma}}{d t} & \geq L \left\lVert \frac{d \thetab_t}{d t} \right\rVert^2 \left( \frac{1}{\innerprod{\thetab_t}{-\mathbf{g}_t^\star}} - \frac{1}{\|\thetab_t\| \left\lVert \frac{d \thetab_t}{d t} \right\rVert} \right) \\
            & = L \frac{\left\lVert \frac{d \thetab_t}{d t} \right\rVert}{\|\thetab_t\|} \left( \frac{1}{\innerprod{\frac{\thetab_t}{\|\thetab_t\|}}{\frac{-\mathbf{g}_t^\star(\thetab_t)}{\left \lVert \mathbf{g}_t^\star \right \rVert_\star}}} - 1 \right).
        \end{split}
    \end{equation}
    We then integrate the two sides from $t_1$ to $t_2 > t_1 > t_0$:
    \begin{equation}
        \int_{t_1}^{t_2} \left( \frac{1}{\innerprod{\frac{\thetab_t}{\|\thetab_t\|}}{\frac{-\mathbf{g}_t^\star}{\left \lVert \mathbf{g}_t^\star \right \rVert_\star}}} - 1 \right) \frac{\left\lVert \frac{d \thetab_t}{d t} \right\rVert}{\|\thetab_t\|} dt \leq \frac{1}{L} \log \frac{\Tilde{\gamma}(t_2)}{\Tilde{\gamma}(t_1)}.
    \end{equation}
    The desired existential statement follows from a proof by contradiction.
\end{proof}




Next, we bound the rate of change of the normalized iterates.

\begin{lemma}\label{lem:rate_of_change_normlzd_iter}
    For almost any $t > 0$, it holds:
    \begin{equation}
        \left \lVert \frac{d \frac{\thetab_t}{\|\thetab_t\|}}{d t} \right\rVert \leq 2 \frac{\left \lVert \frac{d \thetab_t}{dt} \right \rVert}{\|\thetab_t\|}.
    \end{equation}
\end{lemma}

\begin{proof}
    The rate of change of the normalized iterates can be written as follows:
    \begin{equation}
        \begin{split}
            \frac{d \frac{\thetab_t}{\|\thetab_t\|}}{d t} & = \frac{1}{\|\thetab_t\|} \frac{d \thetab_t}{dt} + \thetab_t \left( - \frac{1}{\|\thetab_t\|^2} \frac{d \|\thetab_t\|}{dt} \right) \\
            & = \frac{1}{\|\thetab_t\|} \frac{d \thetab_t}{dt} + \thetab_t \left( - \frac{1}{\|\thetab_t\|^2} \innerprod{\mathbf{n}_t}{\frac{d \thetab_t}{dt}} \right), \;\;\;\;\; (\text{Chain rule})
        \end{split}
    \end{equation}
    where $\mathbf{n}_t \in \partial \|\thetab_t\|$. So, by the triangle inequality, its norm is bounded by:
    \begin{equation}
        \begin{split}
            \left \lVert \frac{d \frac{\thetab_t}{\|\thetab_t\|}}{d t} \right\rVert & \leq  \frac{\left \lVert \frac{d \thetab_t}{dt} \right \rVert}{\|\thetab_t\|} + \frac{1}{\|\thetab_t\|} \left | \innerprod{\mathbf{n}_t}{\frac{d \thetab_t}{dt}} \right | \\
            & \leq 2 \frac{\left \lVert \frac{d \thetab_t}{dt} \right \rVert}{\|\thetab_t\|}. \;\;\;\;\; (\text{definition of dual norm and $\|\mathbf{n}_t \|_\star \leq 1$})
        \end{split}
    \end{equation}
\end{proof}

Finally, we prove an auxiliary lemma that constructs a convenient sequence.

\begin{lemma}\label{lemma:gradient_sequence} For any limit point $\Bar{\thetab}$ of $\left\{\frac{\thetab_t}{\|\thetab_t\|}\right\}_{t \geq 0}$, there exists a sequence $t_n$ with $t_n \rightarrow \infty$ so that $\frac{\thetab_{t_n}}{\| \thetab_{t_n}\|} \rightarrow \bar \thetab$ and $-\frac{\mathbf{g}^{\star}_{t_n}}{\|\mathbf{g}^{\star}_{t_n}\|_{\star}} \rightarrow \mathbf{u}^{\star}$ for some $\mathbf{u}^{\star} \in \partial \norm{\bar \thetab}$.
\end{lemma}
\begin{proof}
Let $\epsilon_n = \frac{1}{n}$ for $n \geq 1$. We construct a sequence $\{t_n\}_{n \geq 1}$, by induction, in the following sense. Suppose $t_1 < \ldots < t_{n-1}$ have been constructed already. Since $\Bar{\thetab}$ is a limit point of the normalized iterates and $\log \Tilde{\gamma}_t \to \log \Tilde{\gamma}_\infty < \infty$ (as $\Tilde{\gamma}_t$ is non-decreasing and bounded from above), there exists $s_n > t_{n-1}$ such that:
\begin{equation}\label{eq:final_bounds}
    \left \lVert \frac{\thetab_{s_n}}{\|\thetab_{s_n}\|} - \Bar{\thetab} \right \rVert \leq \epsilon_n = \frac{1}{n} \;\;\;\;\;\; \text{and} \;\;\;\;\;\; \frac{1}{L} \log \frac{\Tilde{\gamma}_\infty}{\Tilde{\gamma}_{s_n}} \leq \epsilon_n^2 = \frac{1}{n^2}.
\end{equation}
Since $\frac{d \log \|\thetab_t\|}{dt} \leq \frac{\left \lVert \frac{d \thetab_t}{dt} \right \rVert}{\|\thetab_t\|}$ (as in~\Eqref{eq:margin_lemma_eq1}), we have that $\lim_{t \to \infty} \displaystyle\int_{t_A}^{t} \frac{\left\lVert \frac{d \thetab_{t^\prime}}{d {t^\prime}} \right\rVert}{\|\thetab_{t^\prime}\|}dt^\prime \geq \lim_{t \to \infty}\log \frac{\left \lVert \thetab_t \right \rVert}{\left \lVert \thetab_{t_A} \right \rVert} = \infty$ for all $t_A > 0$. Thus, there exists $s_n^\prime > s_n$ such that $\displaystyle\int_{s_n}^{s_n^\prime} \frac{\left \lVert \frac{d \thetab_t}{dt} \right \rVert}{\|\thetab_t\|} dt = \frac{1}{n}$.
Now, from Lemma \ref{lem:alignment_bound}, we know there exists $t_n \in (s_n, s_n^\prime)$ with:
\begin{equation}
    \left( \frac{1}{\innerprod{\frac{\thetab_{t_n}}{\|\thetab_{t_n}\|}}{\frac{-\mathbf{g}_{t_n}^\star}{\left \lVert \mathbf{g}_{t_n}^\star \right \rVert_\star}}} - 1 \right) \leq \frac{1}{L} \frac{\log \frac{\Tilde{\gamma}_{s_n^\prime}}{\Tilde{\gamma}_{s_n}}}{\displaystyle\int_{s_n}^{s_n^\prime} \frac{\left\lVert \frac{d \thetab_t}{d t} \right\rVert}{\|\thetab_t\|}dt} \leq \frac{\frac{1}{n^2}}{\frac{1}{n}} = \frac{1}{n},
\end{equation}
which implies $\innerprod{\frac{\thetab_{t_n}}{\|\thetab_{t_n}\|}}{\frac{-\mathbf{g}_{t_n}^\star}{\left \lVert \mathbf{g}_{t_n}^\star \right \rVert_\star}} \geq \frac{1}{1+\frac{1}{n}} \to 1$ as $n \to \infty$. From the definition of the dual norm, $\innerprod{\frac{\thetab_{t_n}}{\|\thetab_{t_n}\|}}{\frac{-\mathbf{g}_{t_n}^\star}{\left \lVert \mathbf{g}_{t_n}^\star \right \rVert_\star}} \leq 1$, hence $\innerprod{\frac{\thetab_{t_n}}{\|\thetab_{t_n}\|}}{\frac{-\mathbf{g}_{t_n}^\star}{\left \lVert \mathbf{g}_{t_n}^\star \right \rVert_\star}} \to 1$ as $n \to \infty$.
Also, for the normalized iterates we have:
\begin{equation}
    \left \lVert \frac{\thetab_{t_n}}{\|\thetab_{t_n}\|} - \Bar{\thetab} \right \rVert \leq \left \lVert \frac{\thetab_{t_n}}{\|\thetab_{t_n}\|} - \frac{\thetab_{s_n}}{\|\thetab_{s_n}\|} \right \rVert + \left \lVert \frac{\thetab_{s_n}}{\|\thetab_{s_n}\|} - \Bar{\thetab} \right \rVert \stackrel{\text{\Eqref{eq:final_bounds}}}{\leq} \left \lVert \frac{\thetab_{t_n}}{\|\thetab_{t_n}\|} - \frac{\thetab_{s_n}}{\|\thetab_{s_n}\|} \right \rVert + \frac{1}{n}
\end{equation}
To deal with the first term, we can leverage Lemma \ref{lem:rate_of_change_normlzd_iter} which bounds the rate of change of the normalized iterates:
\begin{equation}
    \left \lVert \frac{\thetab_{t_n}}{\|\thetab_{t_n}\|} - \Bar{\thetab} \right \rVert \leq \displaystyle\int_{s_n}^{t_n} \left \lVert \frac{d \frac{\thetab_t}{\|\thetab_t\|}}{d t} \right \rVert dt + \frac{1}{n} \leq 2 \displaystyle\int_{s_n}^{t_n} \frac{\left\lVert \frac{d \thetab_t}{d t} \right\rVert}{\|\thetab_t\|}dt + \frac{1}{n} \leq 2 \displaystyle\int_{s_n}^{s_n^\prime} \frac{\left\lVert \frac{d \thetab_t}{d t} \right\rVert}{\|\thetab_t\|}dt + \frac{1}{n} = \frac{3}{n} \to 0
\end{equation}

From Bolzano-Weierstrass, let $t_{n_\ell}$ be a subsequence of $t_n$ so that $\frac{-\mathbf{g}^{\star}_{t_{n_\ell}}}{\|\mathbf{g}^{\star}_{t_{n_\ell}}\|_{\star}}$ converges to some limit point $\mathbf{u}^{\star}$.  This limit point therefore satisfies $\innerprod{\mathbf{u}^{\star}}{\Bar{\thetab}} = 1, \|\mathbf{u}^{\star}\|_{\star}=1$ which implies $\mathbf{u}^{\star} \in \partial \|\Bar{\thetab}\|$. The required sequence is $t_{n_\ell}$.
\end{proof}

We are now ready to state and prove our main result.
\begin{theorem}\label{thm:main_result}
    For steepest flow (\Eqref{eq:steep_flow}) on the exponential loss, under Assumptions (\ref{ass:1}), (\ref{ass:2}), (\ref{ass:3}), any limit point $\Bar{\thetab}$ of $\left\{ \frac{\thetab_t}{\|\thetab_t\|}\right\}_{t \geq 0}$ is along the direction of a KKT point of the following optimization problem:
    \begin{align}
        \begin{split}
            &\min_{\thetab \in \mathbb{R}^p} \frac{1}{2}\|\thetab\|^2 \\
            \text{s.t. }&  y_i f(\x_i;\thetab) \geq 1, \; \forall i \in [m].
        \end{split}
    \end{align}
\end{theorem}

\begin{proof}
\sloppy

For convenience, denote 
\begin{equation}\label{eq:qmin_def}
q_{\mathrm{min}}(\thetab) = \min_{i \in [m]}y_i f(\mathbf{x}_i;\thetab) = \|\thetab\|^L\gamma(t)    
\end{equation}
We prove that
\begin{equation}
    \hat \thetab := \frac{\Bar{\thetab}}{q_{\mathrm{min}}(\Bar{\thetab})^{1/L}}
\end{equation}
is a KKT point, by showing that $\hat \thetab$ is a $(\varepsilon_\ell, \delta_\ell)$-approximate KKT point (Def.~\ref{def:approx_kkt})
for sequences $\varepsilon_\ell \rightarrow 0, \delta_\ell \rightarrow 0$.

From Lemma \ref{lemma:gradient_sequence}, let $t_n$ be a sequence with $\frac{\thetab_{t_n}}{\|\thetab_n\|} \rightarrow \Bar{\thetab}$ and $-\frac{\mathbf{g}^{\star}_{t_n}}{\left\|{\mathbf{g}^{\star}_{t_n}}\right\|_{\star}} \rightarrow \mathbf{u}^{\star}$ for some $\mathbf{u}^{\star} \in \partial \|\Bar{\thetab}\|$. Denote $\tilde \thetab _{t_n} = \frac{\thetab _{t_n}}{q_{\mathrm{min}}(\thetab _{t_n})^{1/L}}$. Notice that by continuity of $f$,
\begin{equation}
    \lim_{n \rightarrow \infty} \tilde \thetab _{t_n} = \lim_{n \rightarrow \infty}\frac{ \thetab_{t_n}}{\Big(\min_{i \in [m]}y_i f(\mathbf{x}_i;  \thetab_{t_{n}})\Big)^{1/L}} = \lim_{n \rightarrow \infty}\frac{ \frac{\thetab_{t_n}}{\|\thetab_{t_n}\|}}{\frac{1}{\left\|\theta_{t_n}\right\|}\Big(\min_{i \in [m]}y_i f(\mathbf{x}_i;  \thetab_{t_{n}})\Big)^{1/L}} = \hat \thetab.
\end{equation}

By the chain rule, denote $\mathbf{h}_i^{(t_n)} \in \partial f(\mathbf{x}_i; \thetab_{t_n})$ so that 
\begin{equation}
    \mathbf{g}_{t_n}^{\star} = - \sum_{i=1}^m e^{-y_i f(\mathbf{x}_i;\thetab_{t_n})}y_i \mathbf{h}_i^{(t_n)}
\end{equation}
Define $\tilde{\mathbf{h}}_i^{(t_n)} = q_{\mathrm{min}}(\thetab_{t_n})^{\frac1L - 1}\mathbf{h}_i^{(t_n)}$. Since $\tilde \thetab _{t_n} = q_{\mathrm{min}}(\thetab_{t_n})^{-\frac1L}\thetab_{t_n}$, it holds that $\tilde{\mathbf{h}}_i^{(t_n)} =q_{\mathrm{min}}(\thetab_{t_n})^{(-\frac 1L)(L-1)}\mathbf{h}_i^{(t_n)} \in \partial f(\mathbf{x}_i; \tilde{\thetab}_{t_n})$ by Theorem B.2(a) (sub-derivatives of homogeneous functions) in~\citep{LyLi20}. Since $f$ is locally Lipschitz, it holds that $\partial f(\mathbf{x}_i;\cdot)$ is bounded and has a closed graph around $\hat \thetab$, so for every $i \in [m]$ there exists a subsequence $t_\ell$ of $t_n$ with $\tilde{\mathbf{h}}_i^{(t_\ell)} \rightarrow \hat{\mathbf{h}}_i$ for some $\hat{\mathbf{h}}_i \in \partial f(\mathbf{x}_i; \hat \thetab)$. By iteratively choosing subsequences for $i=1,..., m$ we may assume that $\tilde{\mathbf{h}}_i^{(t_\ell)} \rightarrow \hat{\mathbf{h}}_i$ for all $i \in [m]$.

Set $\lambda_{\ell, i} = \frac{\|\thetab_{t_\ell}\|}{\|\mathbf{g}_{t_\ell}^{\star}||_{\star}}q^{1 - \frac2L}_{\mathrm{min}}e^{-y_i f(\mathbf{x}_i;\thetab_{t_\ell})} \geq 0$. Plugging in,
\begin{equation}
    \sum_{i=1}^m\lambda_{\ell, i} y_i \tilde{ \mathbf{h}}_i^{(t_\ell)} = \sum_{i=1}^m\lambda_{\ell, i}q_{\mathrm{min}}^{\frac1L-1}y_i \mathbf{h}_i^{(t_\ell)} = \frac{\|\thetab_{t_\ell}\|}{q^{1/L}_{\mathrm{min}}\|\mathbf{g}_{t_\ell}^{\star}\|_{\star}} \sum_{i=1}^m e^{-y_i f(\mathbf{x}_i;\thetab_{t_\ell})}y_i \mathbf{h}_i^{(t_\ell)} = -\left\|\tilde{ \thetab}_{t_\ell}\right\|\frac{\mathbf{g}_{t_\ell}^{\star}}{\left\|\mathbf{g}_{t_\ell}^{\star}\right\|_{\star}}
\end{equation}

Notice that since $\mathbf{u}^{\star} \in \partial \left\|\Bar{\thetab} \right\| = \partial \left\|\hat \thetab\right\|$, we have $\left\|\hat \thetab\right\|\mathbf{u}^{\star} \in \partial \frac12 \left\|\hat \thetab \right\|^2$ as a primal vector, so we evaluate the following:

\begin{equation}
    \varepsilon_\ell := \left\|\sum_{i=1}^m\lambda_{\ell, i}y_i\hat{\mathbf{h}}_i - \left\|\hat \thetab\right\|\mathbf{u}^{\star}\right\|_2 \leq \left\|\sum_{i=1}^m\lambda_{\ell, i}y_i \left(\hat{\mathbf{h}}_i - \tilde{\mathbf{h}}_i ^{(t_\ell)}\right)\right\|_2 + \left\|\sum_{i=1}^m\lambda_{\ell, i} y_i \tilde{\mathbf{h}}_i ^{(t_\ell)} - \left\|\hat \thetab\right\|\mathbf{u}^{\star}\right\|_2
\end{equation}

The second term goes to 0 since $-\frac{\mathbf{g}^{\star}_{t_\ell}}{\|\mathbf{g}^{\star}_{t_\ell}\|_{\star}} \rightarrow \mathbf{u}^{\star}$. To show the first term goes to 0, we use $\tilde{\mathbf{h}}_i^{(t_\ell)} \rightarrow \hat{\mathbf{h}}_i$ and show $\lambda_{\ell, i}$ are bounded. Recall from Theorem~\ref{thm:margin_monotonicity} that $\left\|{g^{\star}_t}\right\|_{\star} \geq \frac{1}{\left\|\theta_t\right\|}L \mathcal{L} \log \left(\frac{1}{\mathcal{L}}\right)$, and by definition $\log \frac{1}{\cL(\thetab_t)} = \tilde \gamma(\thetab_t)\left \lVert \thetab_t \right \lVert^L$. Also, recall that $\lim_{t \to \infty }\gamma (t)=\lim_{t \to \infty } \tilde\gamma (t) = \tilde \gamma_\infty$. Altogether:
\begin{equation}
    \begin{split}
        \left|\lambda_{\ell, i}\right| & \leq \left\|\thetab_{t_\ell}\right\|^2 \frac{1}{L \cL \log (\frac 1 \cL)} \cdot \left\|\thetab_{t_\ell}\right\|^{L(1-\frac 2L)}\cdot \gamma(t_\ell)^{1-\frac 2L} \cL = \frac{\gamma(t_\ell)^{1-\frac 2L}}{L} \frac{\left\|\thetab_{t_\ell}\right\|^L}{\log \frac 1 \cL} \\
        & = \frac{\gamma(t_\ell)^{1-\frac{2}{L}}}{L} \cdot \frac{\left\|\thetab_{t_\ell}\right\|^L}{\left\|\thetab_{t_\ell}\right\|^L\tilde \gamma({t_\ell})} \\
        & \overset{\ell \rightarrow \infty}\longrightarrow \frac{\tilde \gamma_\infty^{-\frac 2L}}{L} < \infty.
    \end{split}
\end{equation}

For $\delta_\ell$, 
\begin{equation}
    \delta_\ell = \left|\sum_{i=1}^m\lambda_{\ell, i}(y_if(\mathbf{x}_i; \hat{\thetab}) - 1)\right| \leq \left|\sum_{i=1}^m\lambda_{\ell, i}(y_if(\mathbf{x}_i; \hat{\thetab}) - y_i f(\mathbf{x}_i; \tilde{\thetab}_{t_\ell}))\right| + \left|\sum_{i=1}^m\lambda_{\ell, i}(y_if(\mathbf{x}_i;  \tilde{\thetab}_{t_\ell}) - 1)\right|
\end{equation}
The first term goes to 0 from continuity of $f$ and boundedness of $\lambda_{\ell_,i}$. As for the second term, note that from Lemma \ref{lem:feasible}, each summand is non-negative. Plugging in,
\begin{equation}
    \begin{split}
        \sum_{i=1}^m \lambda_{\ell,i} \big(y_if(\mathbf{x}_i; \tilde{\thetab}_{t_\ell})-1\big) &=  \frac{\left\|\thetab_{t_\ell}\right\|}{\left\| \mathbf{g}_{t_\ell}^{\star}\right\|_{\star}}\sum_{i=1}^mq_{\mathrm{min}}^{1 - \frac2L}e^{-y_if(\mathbf{x}_i;\thetab _{t_\ell})}\Big(\frac{y_if(\mathbf{x}_i;\thetab _{t_\ell})}{q_{\mathrm{min}}}-1\Big) \\
        & = \frac{\left\|\thetab_{t_\ell}\right\|}{q^{2/L}_{\mathrm{min}}\left\|\mathbf{g}_{t_\ell}^{\star}\right\|_{\star}}\sum_{i=1}^me^{-y_if(\mathbf{x}_i;\thetab _{t_\ell})}\left(y_if(\mathbf{x}_i;\thetab_{t_\ell}) - q_{\mathrm{min}}\right)
    \end{split}
\end{equation}
Again we use $\left\|\mathbf{g}_{t_\ell}^{\star}\right\|_{\star} \geq \frac{L}{\left\|\thetab_{t_\ell}\right\|}\cL \log \frac{1}{\cL} \geq \frac{L}{||\thetab_{t_\ell}||}e^{-q_{\mathrm{min}}} \log \frac1 \cL$:
\begin{equation}
    \begin{split}
        \sum_{i=1}^m \lambda_{\ell,i}\big(y_i f(\mathbf{x}_i; \tilde{\thetab}_{t_\ell})-1\big) & \leq \frac{\left\|\thetab_{t_\ell}\right\|^2}{q^{2/L}_{\mathrm{min}}Le^{-q_{\mathrm{min}}}\log \frac 1 \cL} \sum_{i=1}^me^{-y_if(\mathbf{x}_i;\thetab_{t_\ell})}\big(y_if(\mathbf{x}_i; \thetab_{t_\ell}) - q_{\mathrm{min}}\big) \\
        & = \frac{\left\|\thetab_{t_\ell}\right\|^2}{q^{2/L}_{\mathrm{min}}L\log \frac 1 \cL} \sum_{i=1}^me^{-(y_if(\mathbf{x}_i;\thetab_{t_\ell}) - q_{\mathrm{min}})}\big(y_if(\mathbf{x}_i; \thetab_{t_\ell}) - q_{\mathrm{min}}\big) \\
         & \underset{\text{\Eqref{eq:qmin_def}}}{=} \frac{1}{ \gamma (t_\ell)^{2/L}L\log \frac 1 \cL} \sum_{i=1}^me^{-(y_if(\mathbf{x}_i;\thetab_{t_\ell}) - q_{\mathrm{min}})}\big(y_if(\mathbf{x}_i; \thetab_{t_\ell}) - q_{\mathrm{min}}\big) \\
        & \underset{\text{Lemma }\ref{lem:smooth_hard_margin}}{\leq} \frac{1}{\tilde \gamma (t)^{2/L}L\log \frac 1 \cL} \sum_{i=1}^me^{-(y_if(\mathbf{x}_i;\thetab_{t_\ell}) - q_{\mathrm{min}})}\big(y_if(\mathbf{x}_i; \thetab_{t_\ell}) - q_{\mathrm{min}}\big)  \\
        & \leq \frac{m}{e\tilde \gamma (t_0)^{2/L}L\log \frac 1 \cL} \\ &\rightarrow 0
    \end{split}
\end{equation}
In the last inequality we used monotonicity of $\tilde \gamma(t)$, and the fact $\forall z \geq0 :z e^{-z} \leq 1/e$. The limit holds since $\cL \overset{t \rightarrow \infty}{\longrightarrow} 0$ and $t_\ell \rightarrow \infty$. By Theorem C.4 in \cite{LyLi20} (which is itself based on a result due to \cite{Dut+13}), we get that $\frac{\Bar{\thetab}}{\left(\min_{i \in [m]} y_i f(\x_i; \Bar{\thetab})\right)^{\frac{1}{L}}}$ is a KKT point of~\Eqref{eq:min_norm_problem_2}.
\end{proof}


\subsection{Generalization to other losses}\label{ssec:gen_loss}

The previous results can be generalized to any loss with exponential tails. In particular, let us proceed to the following definition:
\begin{definition}
    Let $\Phi : \mathbb{R} \to \mathbb{R}$. Assume that $\cL (\thetab) = \sum_{i=1}^m e^{- \Phi\left(y_i f(\x_i; \thetab)\right)}, \thetab \in \mathbb{R}^p$, where $f : \mathbb{R}^d \to \mathbb{R}, y_i \in \{\pm 1\}$. We call the function $l : \mathbb{R} \to \mathbb{R}, l(u) : = e^{- \Phi(u)}$, exponentially tailed, if the following conditions hold:
    \begin{enumerate}
        \item[(i)] $\Phi$ is continuously differentiable.
        \item[(ii)] $\Phi^\prime(u) > 0$ for all $u \in \mathbb{R}$.
        \item[(iii)] The function $g(u) = \Phi^\prime(u) u$ is non-decreasing in $[0, \infty)$.
    \end{enumerate}
\end{definition}
Notice that the definition above covers the exponential loss for $\Phi(u) = u$ and the logistic loss for $\Phi(u) = - \log \log (1+e^{-u})$.
To accommodate different loss functions, Assumption \ref{ass:3} needs to be adjusted as follows:
\begin{itemize}
    \item There is a $t_0 > 0$, such that $\mathcal{L}(\thetab_{t_0}) < e^{- \Phi(0)}$.
\end{itemize}
See Section A in \citep{LyLi20} for a more general, albeit technical, definition that allows the extension of the full analysis to general exponentially-tailed losses.

Under these conditions, we can define the soft margin as follows:
\begin{equation*}
    \Tilde{\gamma} = \frac{l^{-1}(\cL)}{\|\thetab\|^L} = \frac{\Phi^{-1}\left( \log \frac{1}{\cL} \right)}{\|\thetab\|^L},
\end{equation*}
and prove a strict generalization of Theorem~\ref{thm:margin_monotonicity_main}.

\begin{theorem}[Soft margin increases - general loss function]\label{thm:margin_monotonicity_loss_general}
For almost any $t > t_0$, it holds: 
$$\frac{d \log \Tilde{\gamma}}{d t} \geq L \left\lVert \frac{d \thetab}{d t} \right\rVert^2 \left( \frac{\left(\Phi^{-1}\right)^\prime \left( \log \frac{1}{\cL (\thetab_t)}\right)}{L \cL(\thetab_t) \Phi^{-1}\left(\log \frac{1}{\cL(\thetab_t)}\right)} - \frac{1}{\|\thetab_t\|\left\lVert \frac{d \thetab}{d t} \right\rVert}\right) \geq 0.$$
\end{theorem}
\begin{proof}
    Let $\mathbf{n}_t \in \partial \|\thetab_t\|$. We have:
    \begin{equation}\label{eq:margin_lemma_eq1_general}
        \begin{split}
            \frac{d \log \Tilde{\gamma}}{d t} & = \frac{d}{d t} \Phi^{-1}\left( \log \frac{1}{\cL (\thetab_t)}\right) - L \frac{d}{d t} \log \|\thetab_t\| \\
            & = \frac{d}{d t} \Phi^{-1}\left( \log \frac{1}{\cL (\thetab_t)}\right) - L \innerprod{\frac{\mathbf{n}_t}{\|\thetab_t\|}}{\frac{d \thetab}{d t}} \;\;\;\;\; (\text{Chain rule})  \\
            & \geq \frac{d}{d t} \Phi^{-1}\left( \log \frac{1}{\cL (\thetab_t)}\right) - L \frac{\left\lVert \frac{d \thetab}{d t} \right\rVert}{\|\thetab_t\|} \;\;\;\;\; (\text{definition of dual norm and $\|\mathbf{n}_t \|_\star \leq 1$}) \\
            & = - \frac{d \cL (\thetab_t)}{d t} \frac{\left(\Phi^{-1}\right)^\prime \left( \log \frac{1}{\cL (\thetab_t)}\right)}{\cL(\thetab_t) \Phi^{-1}\left(\log \frac{1}{\cL(\thetab_t)}\right)} - L \frac{\left\lVert \frac{d \thetab}{d t} \right\rVert}{\|\thetab_t\|} \;\;\;\;\; (\text{Chain rule}) \\
            & = \left\lVert \frac{d \thetab}{d t} \right\rVert^2 \left( \frac{\left(\Phi^{-1}\right)^\prime \left( \log \frac{1}{\cL (\thetab_t)}\right)}{\cL(\thetab_t) \Phi^{-1}\left(\log \frac{1}{\cL(\thetab_t)}\right)} - \frac{L}{\|\thetab_t\|\left\lVert \frac{d \thetab}{d t} \right\rVert}\right) \;\;\;\;\; (\text{\Eqref{eq:loss_descent}}).
        \end{split}
    \end{equation}
    But, the first term inside the parenthesis can be related to the second one via the following calculation. Recall that, by the chain rule for locally Lipschitz functions (Theorem~\ref{thm:chain_rule}), for any $\mathbf{g}_t \in \partial \cL (\thetab_t)$ there exist $\mathbf{h}_1 \in \partial y_1 f(\x_1; \thetab_t), \ldots, \mathbf{h}_m \in \partial y_m f(\x_m; \thetab_t)$ such that $\mathbf{g}_t = \sum_{i = 1}^m e^{- \Phi(y_i f(\x_i ; \thetab_t))} \Phi^\prime(y_i f(\x_i ; \thetab_t))) \mathbf{h}_i$. Thus, for a minimum norm subderivative $\mathbf{g}_t^\star$, we have:
    \begin{equation}
        \begin{split}
            \innerprod{\thetab_t}{- \mathbf{g}_t^\star} & = \innerprod{\thetab_t}{\sum_{i = 1}^m e^{- \Phi(y_i f(\x_i ; \thetab_t))} \Phi^\prime(y_i f(\x_i ; \thetab_t)) \mathbf{h}_i^\star} \\
            & = \sum_{i = 1}^m e^{- \Phi(y_i f(\x_i ; \thetab_t))} \Phi^\prime(y_i f(\x_i ; \thetab_t)) \innerprod{\thetab_t}{\mathbf{h}_i^\star} \\
            & = L \sum_{i = 1}^m  e^{- \Phi(y_i f(\x_i ; \thetab_t))} \Phi^\prime(y_i f(\x_i ; \thetab_t)) y_i f(\x_i ; \thetab_t),
        \end{split}
    \end{equation}
    where the last equality follows from Euler's theorem for homogeneous functions (whose generalization for subderivatives can be found in Theorem B.2 in \cite{LyLi20}). But, now observe that as per assumption, $u \to \Phi^\prime(u) u$ is non-decreasing and this last term can be lower bounded as:
    \begin{equation}
        \begin{split}
            \innerprod{\thetab_t}{- \mathbf{g}_t^\star} \geq L \sum_{i = 1}^m e^{- \Phi(y_i f(\x_i ; \thetab_t))} \Phi^\prime \left( \Phi^{-1} \left( \log \frac{1}{\cL (\thetab_t)} \right) \right) \Phi^{-1} \left( \log 
            \frac{1}{\cL (\thetab_t)}\right),
        \end{split}
    \end{equation}
    where we used the fact $y_i f(\x_i ; \thetab_t) \leq \Phi^{-1} \left( \log \frac{1}{\cL(\thetab_t)} \right)$ for all $i \in [m]$ (by the monotonicity of $\Phi$ the definition of $\cL$). Leveraging the fundamental property between the derivative of a function and its inverse's, we further get:
    \begin{equation}\label{eq:homog_ineq_gen}
        \begin{split}
            \innerprod{\thetab_t}{- \mathbf{g}_t^\star} \geq L \sum_{i = 1}^m e^{- \Phi(y_i f(\x_i ; \thetab_t))} \frac{\Phi^{-1} \left( \log 
            \frac{1}{\cL (\thetab_t)}\right)}{\left( \Phi^{-1} \right)^\prime \left( \log \frac{1}{\cL(\thetab_t} \right)} = L \cL (\thetab_t) \frac{\Phi^{-1} \left( \log 
            \frac{1}{\cL (\thetab_t)}\right)}{\left( \Phi^{-1} \right)^\prime \left( \log \frac{1}{\cL(\thetab_t} \right)}.
        \end{split}
    \end{equation}
    We have made the first term of~\Eqref{eq:margin_lemma_eq1_general} appear. By plugging~\Eqref{eq:homog_ineq_gen} into~\Eqref{eq:margin_lemma_eq1_general}, we get:
    \begin{equation}
        \begin{split}
            \frac{d \log \Tilde{\gamma}}{d t} & \geq \left\lVert \frac{d \thetab}{d t} \right\rVert^2 \left( \frac{L}{\innerprod{\thetab_t}{-\mathbf{g}_t^\star}} - \frac{L}{\|\thetab_t\|\left\lVert \frac{d \thetab}{d t} \right\rVert}\right) \\
            & \geq \left\lVert \frac{d \thetab}{d t} \right\rVert^2 \left( \frac{L}{\|\thetab_t\| \|\mathbf{g}_t^\star\|_\star} - \frac{L}{\|\thetab_t\|\left\lVert \frac{d \thetab}{d t} \right\rVert}\right) \;\;\;\;\; (\text{definition of dual norm}).
        \end{split}
    \end{equation}
    Noticing that $\|\mathbf{g}_t^\star\|_\star = \left\lVert \frac{d \thetab}{d t} \right\rVert$ (from Proposition~\ref{prop:steep_flow_duality}) concludes the proof.
\end{proof}

\section{Approximate KKT points at finite time through Bregman divergences}\label{app:bregman_measure}

In this section, we further analyze the late-phase geometric properties of the trajectory and, in particular, explain how the alignment $\innerprod{\frac{\thetab_t}{\| \thetab_t\|}}{\frac{- \mathbf{g}_t^\star }{\|\mathbf{g}_t^\star\|_\star}}$ can capture a notion of proximity to stationarity. For algorithms whose norm squared is smooth, this analysis culminates in showing that the (finite-time) iterates are approximate KKT points of the margin maximization problem.

Our analysis uses core ideas from the theory of conjugate functions and Fenchel's duality.
\begin{definition}[Convex conjugate]
    Let $\psi: \mathbb{R}^p \to \mathbb{R}$.
    We denote by $\psi^\star(\cdot)$ the \textit{convex conjugate} of $\psi(\cdot)$:
    \begin{equation}
        \psi^\star (\bm{\omega}) = \sup_{\thetab \in \mathbb{R}^p} \{ \innerprod{\bm{\omega}}{\thetab} - \psi(\thetab) \}.
    \end{equation}
\end{definition}
We will make use of the following properties of conjugate functions.
\begin{proposition}{(Conjugate subgradient theorem - Theorem 23.5 in \cite{Rock70}, Theorem 4.20 in \cite{Beck17})}\label{prop:grad_duality}
    Let $\psi: \mathbb{R}^p \to \mathbb{R}$ be convex and closed. For any $\thetab^\star \in \partial \psi^\star (\thetab)$, it holds $\partial \psi \left( \thetab^\star \right) \ni  \thetab$.
\end{proposition}
\begin{lemma}{(Fenchel-Young inequality) \citep{Fen49}}
    For any $\psi: \mathbb{R}^p \to \mathbb{R}$ and $\bm{\omega}, \thetab \in \mathbb{R}^p$, it holds:
    \begin{equation}\label{eq:fenchel_young}
        \innerprod{\thetab}{\bm{\omega}} \leq \psi(\thetab) + \psi^\star(\bm{\omega}).
    \end{equation}
\end{lemma}

We define the following generalized \textit{Bregman divergence}.
\begin{definition}[Generalized Bregman divergence]\label{def:gen_bregman}
    Let $\psi: \mathbb{R}^p \to \mathbb{R}$ with $\psi(\thetab) = \frac{1}{2} \|\thetab\|_\star^2$ for all $\thetab \in \mathbb{R}^p$. We define the (generalized) Bregman divergence $D_{\frac{1}{2} \|\cdot\|_\star^2}^{\mathbf{m}}(\cdot, \cdot) : \mathbb{R}^p \times \mathbb{R}^p \to \mathbb{R}$ induced by $\psi$ as follows:
    \begin{equation}\label{eq:gen_breg_div}
        D_{\frac{1}{2} \|\cdot\|_\star^2}^{\mathbf{m}} (\mathbf{y}, \mathbf{z}) = \frac{1}{2} \left\lVert \mathbf{y} \right \rVert_\star^2 - \frac{1}{2} \left \lVert \mathbf{z} \right \rVert_\star^2 - \innerprod{\mathbf{m}}{\mathbf{y} - \mathbf{z}},
    \end{equation}
    where $\mathbf{m} \in \partial \frac{1}{2} \left \lVert \mathbf{z} \right \rVert_\star^2$.
\end{definition}
\begin{remark}
Notice that if the function $\psi(\thetab) = \frac{1}{2} \|\thetab\|_\star^2$ is differentiable, then the subdifferental defined at any point collapses to a single element: the gradient of $\psi$. If, further, $\psi$ is strictly convex, then~\Eqref{eq:gen_breg_div} coincides with the usual Bregman divergence induced by $\psi$, defined as $D_\psi(\mathbf{y}, \mathbf{z})=\psi(\mathbf{y})-\psi(\mathbf{z}) - \innerprod{\nabla \psi (\mathbf{z})}{\mathbf{y} - \mathbf{z}}$. Bregman divergences \citep{Bre67} generalize the Euclidean squared distance in different geometries and have found numerous applications in machine learning \citep{NeYu83,Ban+05}.
\end{remark}

Next, we derive guarantees for the finite-time iterates, akin to those of Definition~\ref{def:approx_kkt}, with the difference that the following ones are with respect to a generalized Bregman divergence.

\begin{proposition}\label{prop:approx_bregman_KKT} 
For any $t > t_0$, let $\Tilde{\thetab}_t = \frac{\thetab_t}{\left(\min_{i \in [m]} y_i f(\x_i; \thetab_t)\right)^{\frac{1}{L}}}$. There exist $\lambda_i \geq 0, \Tilde{\mathbf{h}}_i^\star \in \partial f(\x_i; \Tilde{\thetab_t})$, such that for any $\Tilde{\mathbf{k}} \in \partial \frac{1}{2} \|\Tilde{\thetab_t}\|^2$, it holds:
    \begin{equation}\label{eq:approx-epsilon-delta}
        \begin{split}
            D_{\frac{1}{2} \|\cdot\|_\star^2}^{\Tilde{\thetab}_t} \left(\sum_{i = 1}^m \lambda_i y_i \Tilde{\mathbf{h}}_i, \Tilde{\mathbf{k}} \right) & \leq \frac{1}{\Tilde{\gamma}(t_0)^{\frac{2}{L}}} \left( 1 - \innerprod{\frac{\thetab_t}{\| \thetab_t\|}}{\frac{- \mathbf{g}_t^\star }{\|\mathbf{g}_t^\star\|_\star}} \right), \\
            \sum_{i=1}^m \lambda_i \left( y_i f(\x_i; \Tilde{\thetab)} - 1 \right) & \leq \frac{m}{e \Tilde{\gamma}(t_0)^{\frac{2}{L}} L \log \frac{1}{\cL}},
        \end{split}
    \end{equation}
    with $\mathbf{g}_t^\star \in \argmin_{\mathbf{u} \in \partial \cL(\thetab_t)} \|\mathbf{u}\|_\star$.
\end{proposition}
\begin{proof}
    To simplify the notation, let $q_{\text{min}} = \min_{i \in [m]} y_i f(\x_i; \thetab_t)$. Let 
    $\Tilde{\mathbf{k}} \in \partial \frac{1}{2} \|\Tilde{\thetab}_t\|^2$, $\mathbf{g}_t^\star \in \argmin_{\mathbf{u} \in \partial \cL(\thetab_t)} \|\mathbf{u}\|_\star$ and $\mathbf{h}_i^\star \in \partial f(\x_i; \thetab_t), i \in [m],$ such that $\mathbf{g}_t^\star = - \sum_{i = 1}^m e^{- y_i f(\x_i; \thetab_t)} y_i \mathbf{h}_i^\star$ (whose existence is guaranteed from chain rule -- Theorem~\ref{thm:chain_rule_inclusion}). Finally,  we define $\Tilde{\mathbf{h}}_i^\star = q_{\text{min}}^{\frac{1}{L} - 1} \mathbf{h}_i^\star, i \in [m]$, for which it holds: $\Tilde{\mathbf{h}}_i^\star \in \partial f(\x_i; \Tilde{\thetab_t})$ from Theorem B.2(a) in \cite{LyLi20}.
    We set $\lambda_i = \frac{\|\thetab_t\|}{\|\mathbf{g}_t^\star \|_\star} q_{\text{min}}^{1 - \frac{2}{L}} e^{- y_i f(\x_i; \thetab_t)} \geq 0$. Then, it holds:
    \begin{equation}
        \begin{split}
            \sum_{i = 1}^m \lambda_i y_i \Tilde{\mathbf{h}}_i^\star & = \sum_{i = 1}^m \lambda_i q_{\text{min}}^{\frac{1}{L} - 1} y_i \mathbf{h}_i^\star \;\;\;\;\; (\text{Thm B.2(a) in \cite{LyLi20}}) \\
            & = \frac{\|\thetab_t\|}{q_{\text{min}}^{\frac{1}{L}} \|\mathbf{g}_t^\star\|_\star} \sum_{i=1}^m e^{- y_i f(\x_i; \thetab_t)} y_i \mathbf{h}_i
            = - \frac{\|\thetab_t\| \mathbf{g}_t^\star }{ q_{\text{min}}^{\frac{1}{L}} \|\mathbf{g}_t^\star\|_{\star}},
        \end{split}
    \end{equation}
    which is a scaled version of the (minimum norm) subderivative of the loss.
    
    Let $\psi(\thetab_t) = \frac{1}{2} \|\thetab_t\|_\star^2$ be the potential function that we shall use in order to define our divergence. For this specific $\psi$, it holds: $\psi^\star(\bm{\omega}) = \frac{1}{2} \|\bm{\omega}\|^2$ (see for instance Example 3.27 in \cite{BoVa14} for a derivation).
    Recall that in the definition of $D_{\frac{1}{2} \|\cdot\|_\star^2}^{\mathbf{m}}$ (\Eqref{eq:gen_breg_div}) there is an extra choice that we have to make; the one of the subderivative $\mathbf{m}$.
    In what follows, we will specifically measure ``distance'' between $\sum_{i = 1}^m \lambda_i y_i \Tilde{\mathbf{h}}_i$ and $\Tilde{\mathbf{k}}$ using $D_{\frac{1}{2} \|\cdot\|_\star^2}^{\Tilde{\thetab_t}} (\cdot, \cdot)$, i.e. by picking $\mathbf{m} = \Tilde{\thetab_t}$. This is possible, since from Proposition~\ref{prop:grad_duality} it holds that $\Tilde{\thetab_t} \in \partial \frac{1}{2} \|\Tilde{\mathbf{k}}\|_\star^2$. Finally, let $\mathbf{r} \in \partial \|\Tilde{\thetab_t}\|$ be the subgradient of $\lVert \cdot \rVert$ that stems from the chain rule of $\frac{1}{2} \lVert \cdot \rVert^2$ evaluated at $\Tilde{\thetab_t}$.
    We calculate the divergence between the two vectors:
    \begin{equation}\label{eq:divergence_long_calc}
        \begin{split}
            D_{\frac{1}{2} \|\cdot\|_\star^2}^{\Tilde{\thetab_t}} & \left(\sum_{i = 1}^m \lambda_i y_i \Tilde{\mathbf{h}}_i, \Tilde{\mathbf{k}} \right) \\ & = \frac{1}{2} \left \lVert - \frac{\|\thetab_t\| \mathbf{g}_t^\star}{q_{\text{min}}^{\frac{1}{L}} \|\mathbf{g}_t^\star \|_\star} \right \rVert_\star^2 - \frac{1}{2} \left\lVert \Tilde{\mathbf{k}} \right\rVert_\star^2 - \innerprod{\frac{\thetab_t}{q_{\text{min}}^{\frac{1}{L}}}}{-\frac{\|\thetab_t\| \mathbf{g}_t^\star}{q_{\text{min}}^{\frac{1}{L}} \|\mathbf{g}_t^\star \|_\star} - \Tilde{\mathbf{k}}} \\
            & = \frac{1}{2} \frac{\|\thetab_t\|^2}{q_{\text{min}}^{\frac{2}{L}}} - \frac{1}{2} \left \lVert \|\Tilde{\thetab_t}\| \mathbf{r} \right \rVert_\star^2 - \innerprod{\frac{\thetab_t}{q_{\text{min}}^{\frac{1}{L}}}}{\frac{- \|\thetab_t\| \mathbf{g}_t^\star}{q_{\text{min}}^{\frac{1}{L}} \|\mathbf{g}_t^\star \|_\star} - \|\Tilde{\thetab_t}\| \mathbf{r}} \;\;\;\;\; (\text{Chain rule}) \\
            & = \frac{\|\thetab_t\|^2}{q_{\text{min}}^{\frac{2}{L}}} \left( \frac{1}{2} - \frac{1}{2} \left\lVert \mathbf{r} \right\rVert_\star^2 - \innerprod{\frac{\thetab_t}{\|\thetab\|}}{\frac{-\mathbf{g}_t^\star}{\left \lVert \mathbf{g}_t^\star\right \rVert_\star}} + \innerprod{\frac{\thetab_t}{\|\thetab_t\|}}{\mathbf{r}} \right) \\
            & \leq \frac{\|\thetab_t\|^2}{q_{\text{min}}^{\frac{2}{L}}} \left( \frac{1}{2} - \frac{1}{2} \left\lVert \mathbf{r} \right\rVert_\star^2 - \innerprod{\frac{\thetab_t}{\|\thetab_t\|}}{\frac{-\mathbf{g}_t^\star}{\left \lVert \mathbf{g}_t^\star\right \rVert_\star}} + \frac{1}{2} \left\Vert \frac{\thetab_t}{\|\thetab_t\|} \right\rVert^2 + \frac{1}{2} \left\lVert \mathbf{r} \right\rVert_\star^2 \right) \;\;\;\;\; (\text{\Eqref{eq:fenchel_young}}) \\
            & = \frac{\|\thetab_t\|^2}{q_{\text{min}}^{\frac{2}{L}}} \left( 1 - \innerprod{\frac{\thetab_t}{\|\thetab_t\|}}{\frac{-\mathbf{g}_t^\star}{\left \lVert \mathbf{g}_t^\star\right \rVert_\star}}\right) \\ & \leq \frac{1}{\Tilde{\gamma}^{\frac{2}{L}}} \left( 1 - \innerprod{\frac{\thetab_t}{\|\thetab_t\|}}{\frac{-\mathbf{g}_t^\star}{\left \lVert \mathbf{g}_t^\star\right \rVert_\star}}\right) \leq \frac{1}{\Tilde{\gamma}(t_0)^{\frac{2}{L}}} \left( 1 - \innerprod{\frac{\thetab_t}{\|\thetab_t\|}}{\frac{-\mathbf{g}_t^\star}{\left \lVert \mathbf{g}_t^\star\right \rVert_\star}}\right),
        \end{split}
    \end{equation}
    where the last 2 inequalities follow from the relation between soft and hard margin (Lemma \ref{lem:smooth_hard_margin}), and the monotonicity of the former.

    For the second condition, we have:
\begin{equation}\label{eq:2nd_cond_approx}
        \begin{split}
            \sum_{i=1}^m \lambda_i \left( y_i f(\x_i; \Tilde{\thetab)} - 1 \right) & = \frac{\|\thetab\|}{\| \mathbf{g}_t^\star \|_\star} \sum_{i=1}^m q_{\text{min}}^{1 - \frac{2}{L}} e^{- y_i f(\x_i; \thetab)} \left( \frac{y_i f(\x_i; \thetab)}{q_{\text{min}}} - 1 \right) \\
            & = \frac{\|\thetab\| }{q_{\text{min}}^{\frac{2}{L}}\|\mathbf{g}_t^\star \|_\star} \sum_{i = 1}^m e^{-y_i f(\x_i; \thetab)} \left( y_i f(\x_i; \thetab) - q_{\text{min}} \right).
        \end{split}
    \end{equation}
    From~\Eqref{eq:norm_grad_dual_ineq} and~\Eqref{eq:homog_ineq}, we can lower bound the dual norm of the subderivate:
    \begin{equation}
        \| \mathbf{g}_t^\star \|_\star \geq \frac{L}{\| \thetab \|} \cL \log \frac{1}{\cL} \geq \frac{L}{\| \thetab \|} e^{- q_{\text{min}}} \log \frac{1}{\cL}.
    \end{equation}
    By plugging in back to~\Eqref{eq:2nd_cond_approx}, we obtain
    \begin{equation}
        \begin{split}
            \sum_{i=1}^m \lambda_i \left( y_i f(\x_i; \Tilde{\thetab)} - 1 \right) & \leq \frac{\|\thetab\|^2}{q_{\text{min}}^{\frac{2}{L}}L e^{- q_{\text{min}}} \log \frac{1}{\cL}} \sum_{i = 1}^m e^{-y_i f(\x_i; \thetab)} \left( y_i f(\x_i; \thetab) - q_{\text{min}} \right) \\
            & = \frac{\|\thetab\|^2}{q_{\text{min}}^{\frac{2}{L}}L \log \frac{1}{\cL}} \sum_{i = 1}^m e^{- \left(y_i f(\x_i; \thetab) - q_{\text{min}} \right)} \left( y_i f(\x_i; \thetab) - q_{\text{min}} \right) \\ & \leq \frac{1}{ \Tilde{\gamma}(t_0)^{\frac{2}{L}} L \log \frac{1}{\cL}} \sum_{i = 1}^m e^{- \left(y_i f(\x_i; \thetab) - q_{\text{min}} \right)} \left( y_i f(\x_i; \thetab) - q_{\text{min}} \right) \;\;\;\;\; (\text{Lemmata \ref{lem:smooth_hard_margin}, \ref{thm:margin_monotonicity}}) \\
            & \leq \frac{m}{e \Tilde{\gamma}(t_0)^{\frac{2}{L}} L \log \frac{1}{\cL}},
        \end{split}
    \end{equation}
    since the function $u \mapsto e^{-u} u, u > 0$ has a maximum value of $e^{-1}$.
\end{proof}

Building on the previous Proposition, we can provide an explicit characterization of proximity to stationarity at finite time for some algorithmic norms. The proof relies on a fundamental relationship between smoothness of a function and strong convexity of its convex conjugate.
\begin{proposition}{(Conjugate Correspondence Theorem - Thm. 5.26 in \cite{Beck17})}\label{prop:smooth_strcnvx}
    Let $\sigma > 0$. If $\psi$ is a $\frac{1}{\sigma}$-smooth convex function, then its conjugate $\psi^\star$ is $\sigma$-strongly convex.
\end{proposition}

Indeed, we show the following corollary for a special class of steepest flows.

\begin{corollary}
    For steepest flow (\Eqref{eq:steep_flow}) with respect to a norm $\lVert \cdot \rVert$, whose square is a smooth function, on the exponential loss, under assumptions \ref{ass:1}, \ref{ass:2}, \ref{ass:3}, parameter iterates $\thetab_t$ are along the direction of an approximate KKT point 
    of optimization problem~\ref{eq:min_norm_problem_main} for any $t \geq t_0$.
\end{corollary}
\begin{proof}
    \sloppy
    From Proposition~\ref{prop:smooth_strcnvx}, if $\frac{1}{2} \lVert \cdot \rVert^2$ is $\frac{1}{\sigma}$-smooth w.r.t. $\lVert \cdot \rVert$, then the function $\frac{1}{2} \lVert \cdot \rVert_\star^2$ is $\sigma$-strongly convex w.r.t. $\lVert \cdot \rVert_\star$. Thus, the function $D_{\frac{1}{2} \|\cdot\|_\star^2}^{\thetab}$ is defined with respect to a strongly convex function and it becomes a proper Bregman divergence. Hence, from Theorem 5.24 in \cite{Beck17}, for $\Tilde{\mathbf{h}}_i^\star = q_{\text{min}}^{\frac{1}{L} - 1} \mathbf{h}_i^\star$, where $\Tilde{\mathbf{h}}_i^\star \in \partial f(\x_i; \Tilde{\thetab}), i \in [m]$ such that $\mathbf{g}_t^\star = - \sum_{i = 1}^m e^{- y_i f(\x_i; \thetab)} y_i \mathbf{h}_i^\star$ and $\Tilde{\mathbf{k}} \in \partial \frac{1}{2} \|\Tilde{\thetab}\|^2$, it holds:
    \begin{equation}
           D_{\frac{1}{2} \|\cdot\|_\star^2}^{\Tilde{\thetab}} \left(\sum_{i = 1}^m \lambda_i y_i \Tilde{\mathbf{h}}_i^\star, \Tilde{\mathbf{k}} \right) \geq \sigma \left \lVert \sum_{i = 1}^m \lambda_i y_i \Tilde{\mathbf{h}}_i^\star - \Tilde{\mathbf{k}} \right \rVert_\star.
    \end{equation}
    From the equivalence of the norms, this implies that the normalized iterates $\Tilde{\thetab}_{t} = \frac{\thetab_{t}}{\left(\min_{i \in [m]} y_i f(\x_i; \thetab_{t})\right)^{\frac{1}{L}}}$  satisfy all the conditions of an approximate KKT point.
\end{proof}

\begin{remark}
    Leveraging the proof of Theorem~\ref{thm:main_result_main}, and in particular Lemma~\ref{lemma:gradient_sequence}, one can show that there exists a subsequence of the iterates that (a) converges to a limit point, and (b) has vanishing approximation errors $\epsilon(t)$ and $\delta(t)$ as $t \to \infty$. This is because the alignment in the RHS of ~\Eqref{eq:approx-epsilon-delta} goes to 1 for such a subsequence. Therefore, the previous Corollary defines well-motivated approximate KKT points.
\end{remark}

\section{Relationship to Adam and Shampoo}\label{sec:steep_adaptive}

The family of steepest descent algorithms includes simplified versions (momentum turned-off) of two adaptive methods, Adam and Shampoo, which have been very popular for training deep neural networks.

\subsection{Adam}\label{sec:adam}
Adam~\citep{KiBa15} is a popular adaptive optimization method, which is frequently used in deep learning. Following our previous notation, the update rule of Adam amounts to:
\begin{equation}\label{eq:adam}
    \begin{split}
        \mathbf{m}_t & = \beta_1 \mathbf{m}_{t-1} + (1-\beta_1) \nabla \cL (\thetab_{t-1}) \\
        \mathbf{v}_t & = \beta_2 \mathbf{v}_{t-1} + (1 - \beta_2) \nabla \cL (\thetab_{t-1})^2 \\
        \hat{\mathbf{m}}_t & = \frac{\mathbf{m}_t}{1 - \beta_1^t}; \hat{\mathbf{v}}_t = \frac{\mathbf{v}_t}{1 - \beta_2^t} \\
        \thetab_{t} & = \thetab_{t-1} - \eta_t \frac{\hat{\mathbf{m}}_t}{\sqrt{\hat{\mathbf{v}}_t} + \epsilon},
    \end{split}
\end{equation}
where the $\sqrt{}, {}^2, \div$ operations are overloaded to operate elementwise in vectors. Parameters $\beta_1, \beta_2$ control the memory of the update rule, while $\epsilon$ is a numerical precision parameter. Notice that for $\beta_1 = \beta_2 = \epsilon = 0$, we recover sign-gradient descent, a version of \Eqref{eq:steep_desc} with $\lVert \cdot \rVert = \lVert \cdot \rVert_\infty$ where $\Delta \thetab_t$ is normalized to satisfy $\left \lVert \Delta \thetab_t\right \rVert = 1$.

\citet{Wan+22} studied the implicit bias of (\ref{eq:adam}) for $\epsilon > 0$ in linear networks establishing bias towards $\ell_2$ margin maximization, while \citet{ZZC24} analyzed the case of $\epsilon=0$ and generic $\beta_1, \beta_2 \in [0, 1)$ also in linear networks and showed bias towards $\ell_\infty$ margin maximization.

\subsection{Shampoo}\label{ssec:shampoo}

Shampoo~\citep{Gup+18} is an adaptive optimization algorithm, which has recently gained popularity in deep learning applications. For each weight matrix $\mathbf{W}_t$ and its corresponding gradient matrix $\mathbf{G}_t$, the update rule amounts to:
\begin{equation}
    \begin{split}
    \mathbf{L}_{t} & = \mathbf{L}_{t-1} + \mathbf{G}_t \mathbf{G}_t^T \\
    \mathbf{R}_{t} & = \mathbf{R}_{t-1} + \mathbf{G}_t^T \mathbf{G}_t \\
    \mathbf{W}_{t+1} & = \mathbf{W}_t - \eta \mathbf{L}_t^{-1/4} \mathbf{G}_t \mathbf{R}_t^{-1/4}.
    \end{split}
\end{equation}

As \citet{BeNe24} recently observed, with momentum turned off, this simplifies to:
\begin{equation}
    \begin{split}
        \mathbf{W}_{t+1} & = \mathbf{W}_{t} - \eta_t \mathbf{U}_t \mathbf{V}_t^T ,
    \end{split}
\end{equation}
where $\mathbf{U}_t, \mathbf{V}_t$ contain the left and right singular vectors of $\mathbf{G}_t$, i.e., $\mathbf{G}_t = \mathbf{U}_t \mathbf{\Sigma}_t \mathbf{V}_t$. \citet{BeNe24} further remarked that this update corresponds to steepest descent (in matrix space) with respect to the \textit{spectral norm} $\sigma_{\mathrm{max}}(\cdot)$. This is equivalent to an architecture-dependent norm in parameter space. For instance, if $\thetab = (\mathbf{W}_1, \ldots, \mathbf{W}_L)$, then Shampoo without momentum corresponds to steepest descent with respect to the norm $\|\thetab\|_S := \max\left(\sigma_{\mathrm{max}}(\mathbf{W}_1), \ldots, \sigma_{\mathrm{max}}(\mathbf{W}_L)\right)$.

\section{Experimental Details}\label{app:expr_details}

All experiments are implemented in PyTorch \citep{Pas+17}.

\paragraph{Teacher-student experiments}
We use the following hyperparameters: $d = 2^{32}, k = 64, k^\prime = 1024, m = 250$,  learning rate $\eta = 6\times10^{-3}$ and density $\frac{\|\thetab^\star\|_0}{k^\prime (d+1)} = 0.0001$ (3 coordinates active per neuron). We vary the scale of initialization in $\{0.1, 0.01, 0.001\}$ and we train for $10^5$ epochs. Each random seed affects the draw of the datasets and the initialization of the parameters of the network. Test accuracy is estimated using 20,000 unseen data drawn from the same generative process.
\begin{figure}
    \centering
    \includegraphics[width=0.47\linewidth]{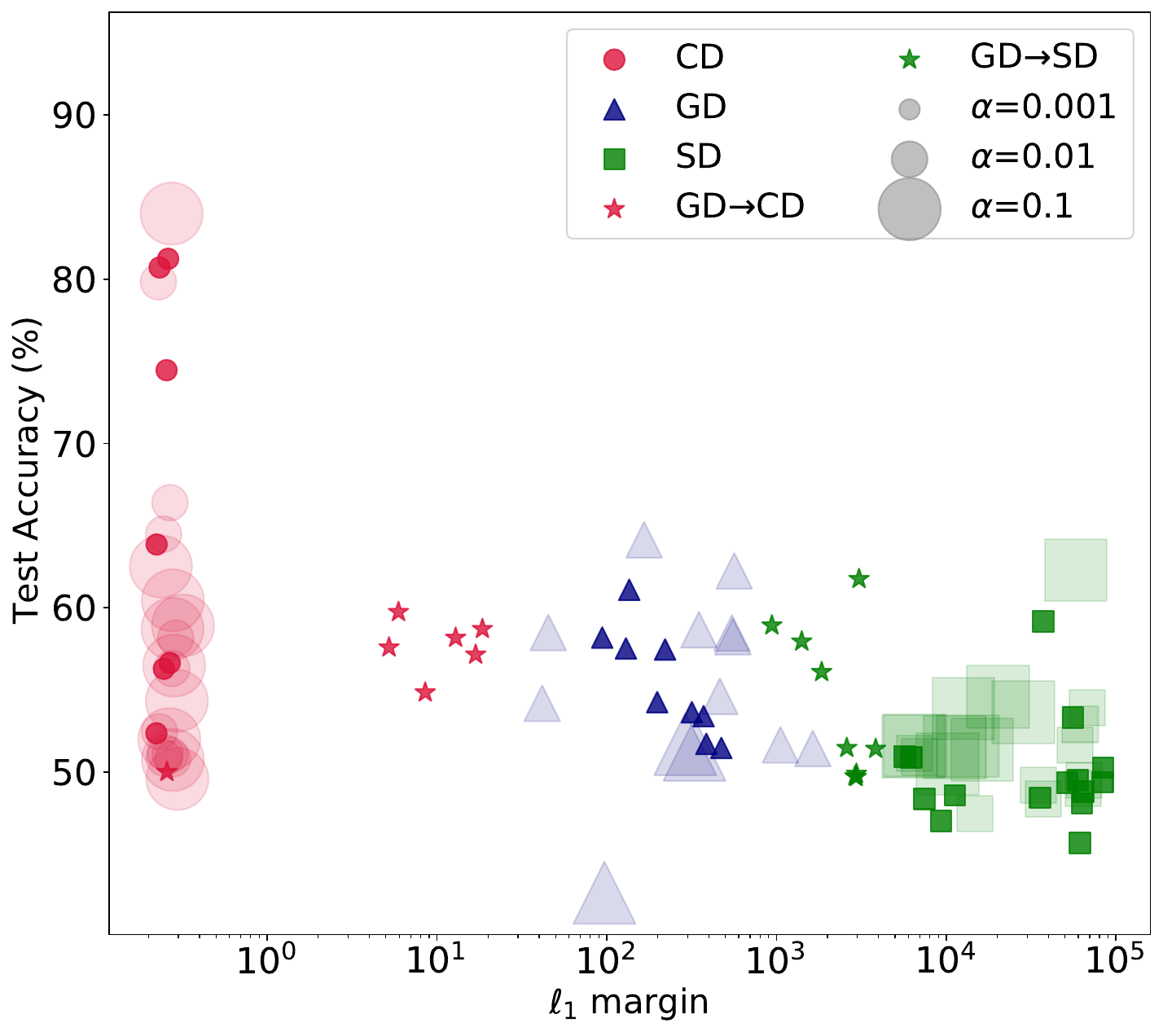}
    \includegraphics[width=0.47\linewidth]{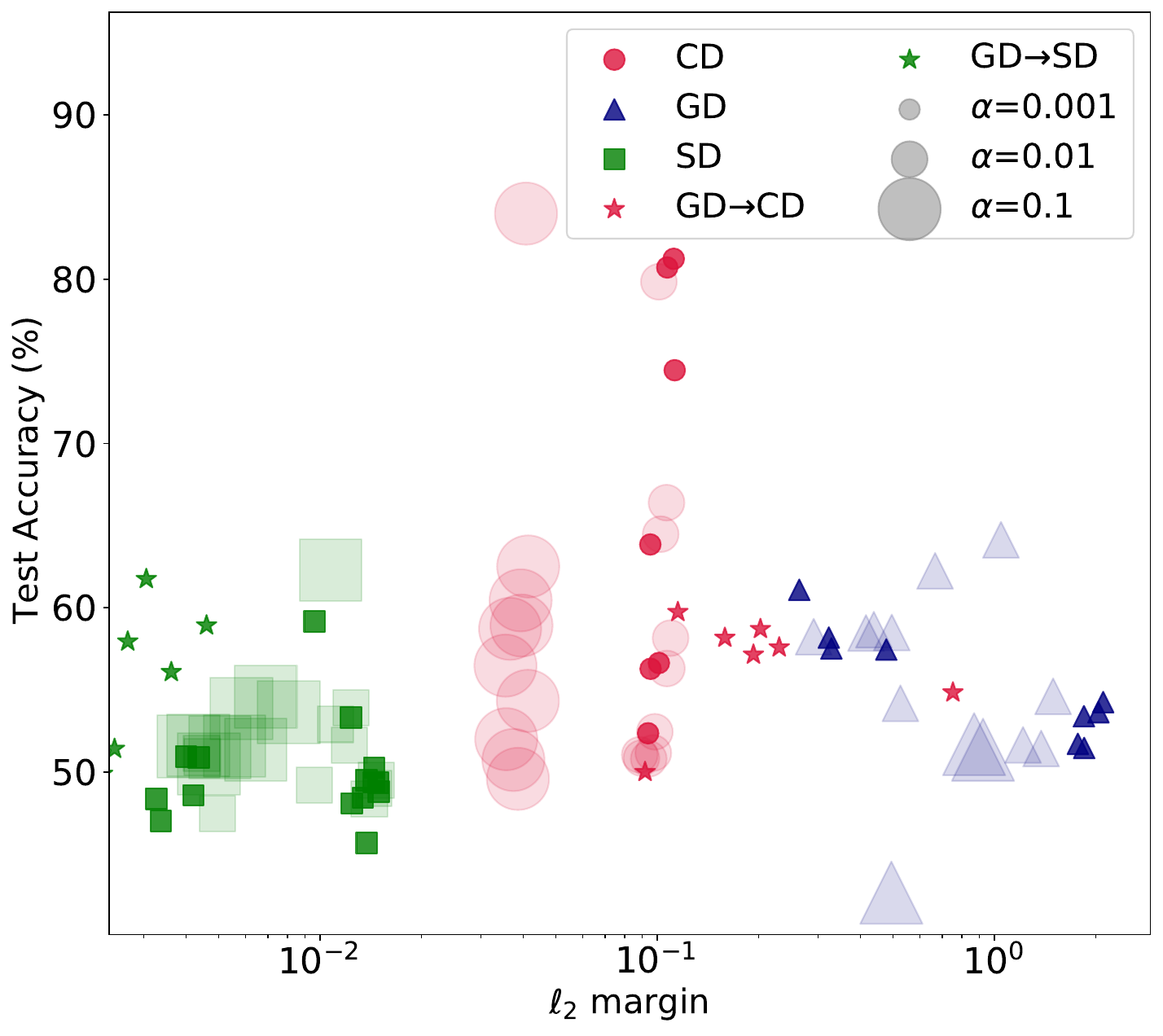}
    \caption{\textbf{Geometric margins vs test accuracy in a teacher-student setup.} \textit{Left:} $\ell_1$ margin. \textit{Right:} $\ell_2$ margin. Each point corresponds to a different run (different random seed).}
    \label{fig:add_margin}
\end{figure}

\paragraph{Adam experiments}
We use a constant learning rate of $3 \times 10^{-3}$ and 1-hidden layer neural networks of width $128$, optimizing the logistic loss. The digits that we extract are '3' and '6' (100 training points). Each random seed corresponds to a different draw of the training dataset and different initialization. Sign gradient descent runs were very effective in minimizing the training loss, and we stopped the training early after the loss reached value smaller than $10^{-7}$ in order to avoid numerical issues. We depict the final value, repeated for as many epochs as shown in the figures (as if the model has indeed converged).

Figure~\ref{fig:adam_sd27} shows accuracy and margins for a different pair of digits ("2" vs "7").

\paragraph{Shampoo experiments} We trained 1-hidden layer neural networks of width $128$ with the second layer frozen at initialization. The scale of initialization was set to $10^{-2}$. We used a learning rate of $10^{-2}$ for \texttt{Adam} and $10, 20$ for \texttt{GD} and \texttt{Shampoo}, respectively (the unusually large learning rate is due to a frozen second layer -- the models were still away from the EOS regime).

\begin{figure}
    \centering
    \includegraphics[scale=0.4]{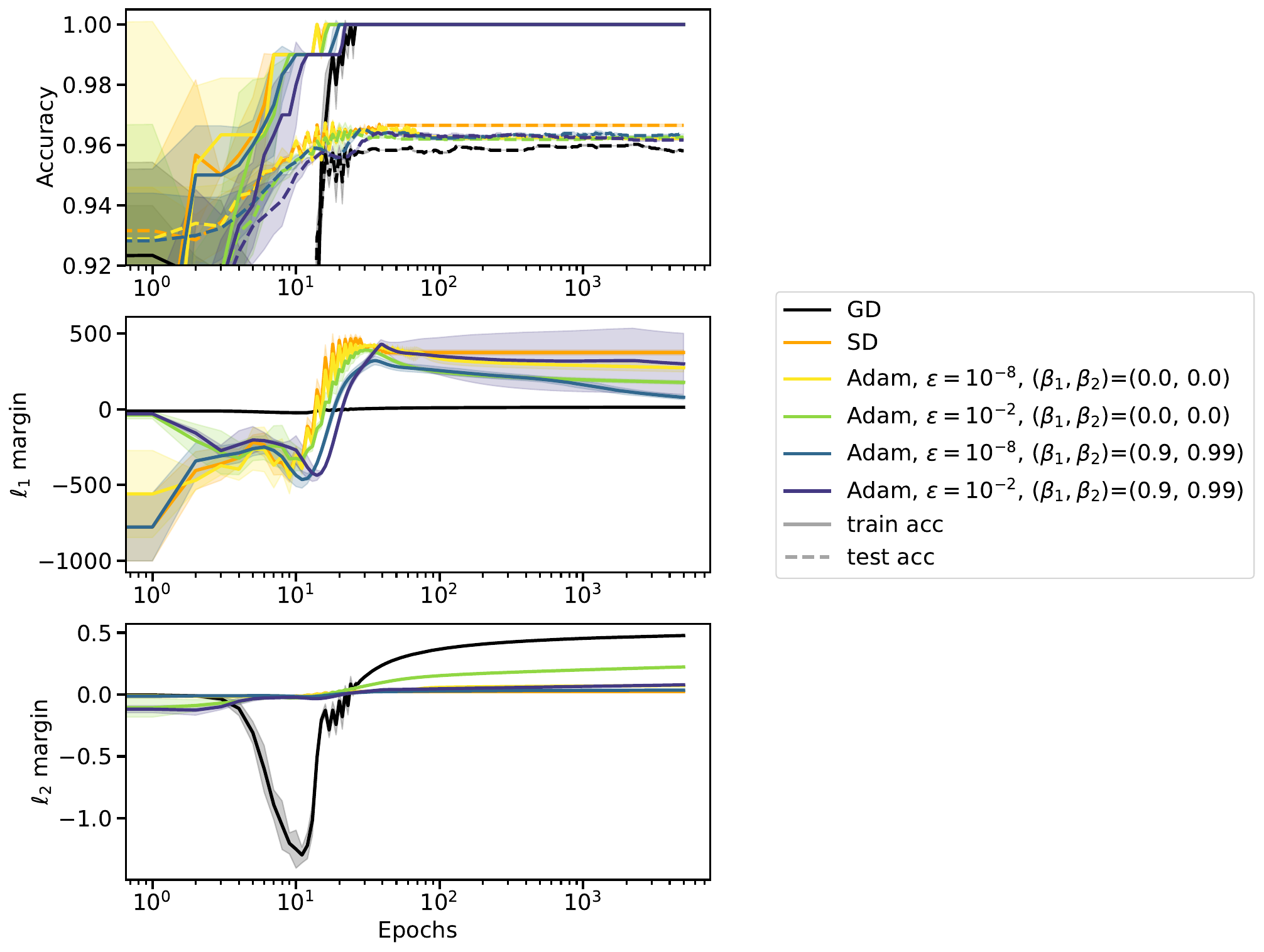}
    \caption{\textbf{Relationship between Adam and steepest descent algorithms.} Digits '2' and '7'.}
    \label{fig:adam_sd27}
\end{figure}

\end{document}